%% file: Template-Neurips-26/neurips_2026.tex
\documentclass{article}

\usepackage[preprint]{neurips_2026}

\usepackage[utf8]{inputenc} 
\usepackage[T1]{fontenc}    
\usepackage{hyperref}       
\usepackage{url}            
\usepackage{booktabs}       
\usepackage{amsfonts}       
\usepackage{nicefrac}       
\usepackage{microtype}      
\usepackage[dvipsnames]{xcolor}         
\usepackage{subcaption}
\usepackage{comment}
\usepackage{graphicx}
\usepackage{multirow}
\input{math_commands}

\input{macros}
\usepackage[capitalize]{cleveref}
\usepackage{enumitem}
\usepackage{wrapfig,lipsum,booktabs}
\usepackage{adjustbox}
\definecolor{darkblue}{rgb}{0, 0, 0.5}
\hypersetup{colorlinks=true, citecolor=darkblue, linkcolor=darkblue, urlcolor=darkblue}

\title{
RoPE Distinguishes Neither Positions Nor Tokens \\ in Long Contexts, Provably 
}

\author{
Yufeng Du \\
University of Illinois at Urbana-Champaign \\
USA \\
\texttt{yufengd4@illinois.edu} \\
\And
Phillip Harris \\
University of Bonn \\
Germany \\
\AND
Minyang Tian \\
University of Illinois at Urbana-Champaign \\
USA \\
\And
Eliu A Huerta \\
Argonne National Laboratory \\
USA
\AND
Srikanth Ronanki \\
Amazon AGI \\
USA \\
\And 
Subendhu Rongali \\
Amazon AGI \\
USA \\
\And
Aram Galstyan \\
Amazon AGI \\
USA \\
\AND
Hao Peng \\
University of Illinois at Urbana-Champaign \\
USA \\
\texttt{haopeng@illinois.edu} 
}

\begin{document}

\maketitle

\input{content/000_abstract}

\input{content/100_intro}
\input{content/n2_demystifying_rope}

\input{content/n3_failure_modes}

\input{content/211_evidence_real_model}
\input{content/300_discussions}

\input{content/limitations}

\input{content/acknowledgement}
\bibliography{ref.bib}
\bibliographystyle{Template-colm-2026/colm2026_conference}

\input{content/400_appendix}
\input{content/101_related_works}

\end{document}

%% file: math_commands.tex
\usepackage{amsmath,amsfonts,bm, amssymb}
\usepackage{amsthm}

\newtheorem{thm}{Theorem}

\theoremstyle{remark}
\newtheorem{rmk}{Remark}[section]

\newtheorem{definition}[rmk]{Definition}

\def\eqref#1{equation~\ref{#1}}

\def\1{\bm{1}}

\DeclareMathAlphabet{\mathsfit}{\encodingdefault}{\sfdefault}{m}{sl}
\SetMathAlphabet{\mathsfit}{bold}{\encodingdefault}{\sfdefault}{bx}{n}



%% file: macros.tex
\definecolor{orange}{rgb}{1,0.5,0}
\definecolor{mdgreen}{rgb}{0.05,0.6,0.05}
\definecolor{mdblue}{rgb}{0,0,0.7}
\definecolor{dkblue}{rgb}{0,0,0.5}
\definecolor{dkgray}{rgb}{0.3,0.3,0.3}
\definecolor{slate}{rgb}{0.25,0.25,0.4}
\definecolor{gray}{rgb}{0.5,0.5,0.5}
\definecolor{ltgray}{rgb}{0.7,0.7,0.7}
\definecolor{purple}{rgb}{0.7,0,1.0}
\definecolor{lavender}{rgb}{0.65,0.55,1.0}

\definecolor{mypurple}{RGB}{111,61,121}
\definecolor{myblue}{RGB}{46,88,180}
\definecolor{myred}{RGB}{181,68,106}
\definecolor{myyellow}{RGB}{204,143,55}

\newcommand{\term}[1]{\textbf{#1}} 

\DeclareSymbolFont{extraup}{U}{zavm}{m}{n}
\DeclareMathSymbol{\vardiamond}{\mathalpha}{extraup}{87}

\usepackage[most,skins,theorems]{tcolorbox}

\tcbset{
  aibox/.style={
    width=\linewidth,
    top=7pt,
    bottom=2pt,
    left=2pt,
    right=2pt,
    colback=blue!6!white,
    colframe=black,
    colbacktitle=black,
    enhanced,
    center,
    attach boxed title to top left={yshift=-0.1in,xshift=0.1in},
    boxed title style={boxrule=0pt,colframe=white,},
  }
}
\newtcolorbox{AIbox}[2][]{aibox,title=#2,label=#1}

\newcounter{hyp}
\tcbset{
  rqbox/.style={
    width=\linewidth,
    top=7pt,
    bottom=2pt,
    left=2pt,
    right=2pt,
    colback=red!6!white,
    colframe=black,
    colbacktitle=black,
    enhanced,
    center,
    attach boxed title to top left={yshift=-0.1in,xshift=0.1in},
    boxed title style={boxrule=0pt,colframe=white,},
  }
}
\newtcolorbox{RQbox}[2][]{rqbox,title=#2,before upper={\refstepcounter{hyp}\label{#1}}}

%% file: content/000_abstract.tex
\begin{abstract}

We identify intrinsic limitations of Rotary Positional Embeddings (RoPE) in Transformer-based long-context language models. 
Our theoretical analysis abstracts away from the specific content of the context and depends \emph{only} on its length.
We prove that as context length increases, RoPE-based attention becomes unpredictable and loses two properties that are central to its effectiveness. First, it loses its locality bias: RoPE is no more likely to favor nearer positions than substantially farther ones. Second, it loses consistency in token relevance: a key vector that receives a higher attention score than an alternative at one position may receive a lower score at another. In both cases, the probability of failure approaches 0.5, no better than random guessing. We further prove that the attention score can remain unchanged when a key token is moved to a different position, or even replaced by a different token, indicating a failure to distinguish positions or tokens. Adjusting the RoPE base trades off distinguishing positions against distinguishing tokens but cannot preserve both at the same time. Increasing the RoPE base hyperparameter, a common practice in today’s long-context models, helps distinguish different tokens, but inevitably sacrifices the ability to distinguish positions. 
Our empirical analysis shows that multi-head, multi-layer architectures are insufficient to overcome these limitations. Our findings suggest that fundamentally new mechanisms for encoding position and token order may be needed in future Transformer long-context language models.

\end{abstract}

%% file: content/100_intro.tex
\section{Introduction}

Positional embeddings are essential in Transformers, because attention is otherwise permutation-invariant and cannot distinguish token order \citep{vaswani2017attention}. Among the many positional embeddings, Rotary Positional Embedding (RoPE, \citealp{su2021roformer}) has emerged as the de facto choice in modern Transformer-based large language models (LLMs). The popularity of RoPE emerges from several appealing properties.
Through rotary operations, RoPE encodes the relative distance between tokens, and induces a locality bias that favors nearby tokens over distant ones \citep{su2021roformer}.
Such inductive biases align well with the structure of natural language and prove beneficial for both training convergence \citep{gelberg2025extendingcontextpretrainedllms} and extension to longer context lengths \citep{press2022train}.

Despite the increasing advertised context lengths of recent LLMs \citep{10.5555/3692070.3692634,geminiteam2024gemini15unlockingmultimodal,deepseekai2026deepseekv4}, many recent studies show that these models often struggle with long-context tasks that should be well within their capabilities, even at input lengths well within their claimed context lengths \citep{liu-etal-2024-lost, hsieh2024ruler, kuratov2024babilong, du2025contextlengthhurtsllm}. These recurring failures beg a fundamental question: \emph{Are these failures artifacts of engineering choices, or do they reflect intrinsic limitations of RoPE itself?} Answering this question is important because it determines whether future progress in long-context Transformers should focus primarily on improved engineering, or instead require fundamentally new 
new mechanisms for encoding positions and token order.

Our answer is that RoPE itself has intrinsic limitations in long contexts. 
We systematically explain this with a theoretical analysis of \textit{single-head attention}
that abstracts away from the specific content of the context and depends \emph{only} on its length.\footnote{We provide a discussion for the multihead and multilayer case in \cref{sec:real_model_discussions}.}
We show under mild assumptions\footnote{See \S\hyperref[sec:limitations]{Limitations}.} that as context length increases, RoPE’s effect on attention becomes increasingly unpredictable and undermines the very properties that make it effective in language models, struggling on two primary objectives:
\begin{itemize}[noitemsep,topsep=0pt,parsep=0pt,partopsep=0pt]
    \item First, RoPE fails to distinguish positions (\S\ref{sec:pos}). As the context length grows, the same token may receive a higher attention score at a farther position than at a closer one, with probability approaching 0.5 (\term{position inversion}; \S\ref{sec:position_inversion}). RoPE thus becomes no better than random chance at favoring nearer positions over farther ones, effectively losing its locality inductive bias.
    We further identify a specific failure mode, which we call \term{position aliasing}: for a fixed query and key, moving the key to a different position may leave its attention score unchanged, so the model no longer distinguishes positions reliably (\S\ref{sec:position_aliasing}; \cref{fig:att_inv_ill}). 
    \item Second, RoPE fails to distinguish tokens (\S\ref{sec:tok}). 
    As the context length grows, the relative ranking of two different key tokens for a given query, reflected by the attention scores they receive, can be arbitrarily reversed across positions: a token ranked above another at one position may be ranked below it at another (\term{token inversion}; \S\ref{sec:token_inversion}).
    The probability of token inversion also approaches 0.5, no better than random chance.
    Moreover, longer context induces a phenomenon we call \term{token aliasing}: for a fixed query and key position,
    replacing the key token with a different token   may leave the attention score unchanged, so the model effectively fails to distinguish tokens reliably (\S\ref{subsec:token_aliasing}).
\end{itemize}
The above theoretical results are derived from a key new insight in our analysis, which treats the unnormalized attention score as a normal random variable (\S\ref{sec:rope_as_normal}).

\begin{figure*}[t]
    
    \centering
    \includegraphics[width=0.9\textwidth]{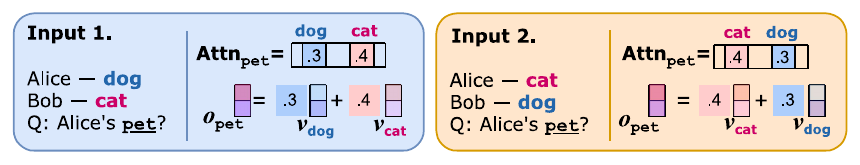}
    \caption{Position aliasing induces an attention invariance failure: there exist large numbers of  positions where swapping two key tokens (\texttt{dog}, \texttt{cat}) keeps the attention output of a query token $o_{\texttt{pet}}$ unchanged. 
    }
    \label{fig:att_inv_ill}
\end{figure*}

Our empirical analysis on Llama 3.1-8B~\citep{grattafiori2024llama3herdmodels}, which has a claimed context length of 128K tokens,
confirms our theoretical conclusions about position and token inversion.
It further shows that both position aliasing and token aliasing occur ubiquitously: across a context length of only 8K tokens, a staggering 75K pairs of positions exhibit position aliasing, appearing regardless of positional proximity; additionally, around 150 positions exhibit token aliasing in this range.
Our theory suggests that commonly used length-extension techniques do \emph{not} resolve the problem. Adjusting the RoPE base hyperparameter trades off the two failure modes rather than eliminating them. In particular, increasing the RoPE base helps preserve consistency in token relevance, but weakens the ability to distinguish positions.

Our experiments confirm that these failures persist in real multihead, multilayer LLMs (\S\ref{sec:exp}). We tested 6 popular models from 7B to over 100B on a simple task: given a list, the model must identify the value at the $k$-th position. This task addresses the ability to distinguish position, rather than distinguish token identities, since modern LLMs are commonly optimized for the latter through retrieval-style objectives \citep{kamradt2023needleinahaystack}.  
With just 4 distinct values in the list, all models perform no better than random guesses in as short as 4K tokens, a length disproportional to what these models were trained on. This strengthens our theoretical analysis of the single-head case by showing that the same positional failure persists in practical models. 

Our findings temper some of the recent optimism created by rapidly increasing advertised context lengths. Extending the nominal context length alone is flawed if the underlying positional mechanism degrades as the context length grows. Our analysis provides a mechanistic explanation for the recurring long-context failures observed in recent studies \citep{liu-etal-2024-lost, hsieh2024ruler, kuratov2024babilong, du2025contextlengthhurtsllm}, suggesting that the gap between the nominal context limit and reliable use of distant information may not be eliminated through better data or engineering alone; instead, they reflect the fundamental limitations of the positional mechanism. By identifying such limitations, this work motivates further study into fundamentally new approaches to positional mechanisms better suited to long-context language modeling.

%% file: content/n2_demystifying_rope.tex
\section{Demystifying RoPE}
\label{sec:demystifying_rope}

Attention in transformers should achieve two objectives:
(1) \textbf{Position identification}, to 
encode where a token occurs in the text and allow attention to distinguish positions and capture contextual dependencies shaped by word order.
Failures hurt the model's ability to understand the context dependency and lead to errors in tasks like counting or reasoning.
(2) \textbf{Token identification}, to have each query  distinguish among tokens and identify those that are contextually salient. 
Failures cause the model to ignore relevant inputs and generate hallucinated content.
Long-context tasks often require a combination of these two objectives \citep{vaswani2017attention,liu-etal-2024-lost,bai2024longbenchbilingualmultitaskbenchmark}.

We define the \textbf{RoPE product} as the un-normalized attention score, i.e. the dot product between a query and a key after RoPE has been applied to both.
This section aims to address two questions through the lens of the RoPE product:
How does the RoPE product help with position and token identification (\S\ref{sec:background})?
How does RoPE-based attention behave as the context length increases (\S\ref{sec:rope_as_normal})? 
We answer both through our key insight of treating the RoPE product as a normal random variable.

Throughput the paper, our theoretical analysis abstracts away from the specific content of the context and considers its length alone.

\subsection{Background}\label{sec:background}

\label{par:base_wavelength}
For a pair of query and key vectors $\mathbf{q}$ and $\mathbf{k}$, RoPE \citep{su2021roformer} divides the $d$ hidden dimensions into $h=d/2$ pairs of 2D vectors. As the token position changes, each 2D vector rotates at an angular frequency that is distinct to its dimension pair. 
The dot product between $\mathbf{q}$ and $\mathbf{k}$ after applying RoPE to both (the RoPE product) can be written as a function of their relative distance, $m$:
\begin{align*}
S(m)=S_{\mathbf{q}, \mathbf{k}}(m)=\sum_{n=0}^{h-1} a_n \cos(m\theta^n+\phi_n).
\end{align*}
The base frequency is $\theta=B^{-1/h}\in(0,1)$ where $\Theta(B) > M$
 is the RoPE base.\footnote{
Following standard practice we assume $M<\Theta(B)$,
since otherwise even the lowest frequency term oscillates and loses its uniqueness \citep{liu2026rotarypositionalembeddingsphase} The factor can be $2\pi$, $\pi$ or other values under different criteria. 
} Vectors $\mathbf{a}$ and $\bm{\phi}$ are determined solely by $\mathbf{q}$ and $\mathbf{k}$.
For the $n$-th frequency component, its amplitude $a_n>0$ is the product of the norms of the corresponding 2D vectors $(q_{2n}, q_{2n+1})$ and $(k_{2n}, k_{2n+1})$, and its phase $\phi_n \in [0, 2\pi)$ is the  angle subtended by them.

\paragraph{High-frequency components oscillate; low-frequency components decay}

For a context length limit of $M$, one typical way of analyzing the RoPE product is to separate the high and low frequency components using the threshold value $\lambda(M)=\Theta(h\log_B M)$ \citep{jonasson2025rotaryoffsetfeatureslarge, liu2023scaling, peng2024yarn, miranda2024round}. 
For $m\in [0, M)$,
\term{high-frequency} components complete at least one circle around the origin with $n\ll \lambda(M)$;
\term{low-frequency} ones only rotate a small angle with $n\gg \lambda(M)$.

\begin{figure*}[tbhp]
    
    \centering
    \begin{subfigure}[t]{0.31\textwidth}
        \centering
        \includegraphics[width=1.06\textwidth]{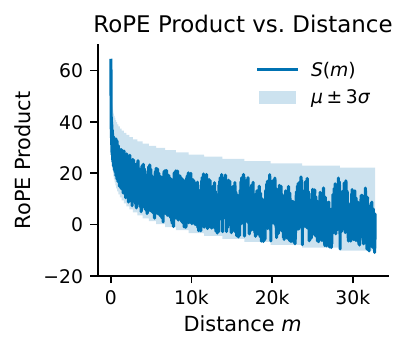}
        \caption{
        Overall pattern of RoPE waveform.
         Shadow shows estimation of $S(m)$ as a normal. 
        }
        \label{fig:rope_product_vs_distance_a}
    \end{subfigure}
    ~ 
    \begin{subfigure}[t]{0.31\textwidth}
        \centering
        \includegraphics[width=1.06\textwidth]{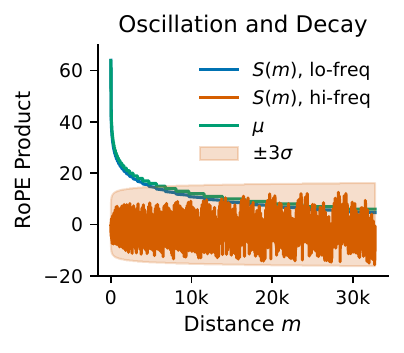}
        \caption{{\color{blue} Low} and {\color{RedOrange} high} frequency parts of $S(m)$, and {\color{ForestGreen} mean} and {\colorbox{Apricot}{standard deviation}} as a normal. 
        }
        \label{fig:rope_product_vs_distance_b}
    \end{subfigure}
    ~
    \begin{subfigure}[t]{0.31\textwidth}
            \centering
        \includegraphics[width=1.06\textwidth]{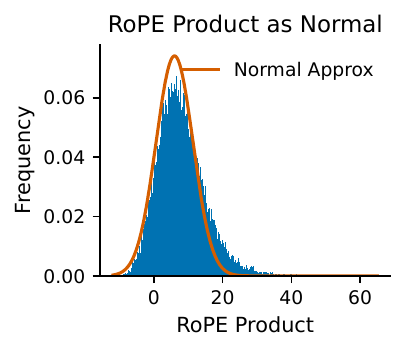}
        \caption{ {\color{blue} Full distribution} of $S(m)$ over $m\in [0, M)$.  }
        \label{fig:rope_product_as_normal}
    \end{subfigure}
    \caption{Illustration of the RoPE product and its normal approximation when $S(m)=\sum_{n}\cos(m\theta^n)$, $m\in[0, 32,768), h=64, B=10^5.$ 1k = 1,000.
    }
    \label{fig:rope_product_vs_distance}
\end{figure*}

\cref{fig:rope_product_vs_distance_b} illustrates the oscillation of the high-frequency components
and the decay effect of the low-frequency ones.\footnote{Strictly speaking, decay is not guaranteed to occur, but is nevertheless preferred. See \cref{par:decay} for a more in-depth discussion.}
RoPE helps with the two primary objectives discussed earlier.
For position identification, high-frequency oscillation helps capture the difference between close positions, while the low-frequency decay globally distinguishes distant position pairs, promoting a locality inductive bias.
For token identification, low-frequency components play a stabilizing role: their slower rotations preserve the relative ordering of token relevance, as they are less perturbed by relative distances.

\subsection{Key Insight: The RoPE Product As a Normal Random Variable}\label{sec:rope_as_normal}

Previous work has largely focused on low-frequency decay due to its analytical tractability \citep{miranda2024round,NEURIPS2024_9f12dd32,xiong-etal-2024-effective}.
We develop a probabilistic characterization of the distributional behavior of the RoPE product. 
This perspective yields a deeper understanding of RoPE's behavior.
A core theoretical contribution of this paper can be informally stated as follows:
\begin{rmk}\label{rmk:normal}
If the distance $m$ between a query $\mathbf{q}$ and a key $\mathbf{k}$ is randomly sampled from any interval $[A, M)$, where $M-A$ is large, then the RoPE product $S_{\mathbf{q},\mathbf{k}}(m)$ can be modeled as a normal random variable
$$
\widetilde{S}=\widetilde{S}_{[A,M)}({\mathbf{q},\mathbf{k}})\sim N(\mu_M({\mathbf{q},\mathbf{k}}), \sigma_M^2({\mathbf{q},\mathbf{k}})),
$$
with its mean decided by its low frequency terms, and its variance decided by its high frequency terms. The high frequency threshold is determined by the context limit, $M$. 
\end{rmk}

Remark~\ref{rmk:normal} follows from an application of the Central Limit Theorem.
See \cref{sec:rope-sum-as-normal} 
for details and 
\cref{fig:rope_product_as_normal} for empirical validation.
Remark~\ref{rmk:normal} provides a powerful tool to characterize the behavior of RoPE product $S(m)$: it behaves approximately as a normal variable whose mean decreases (decay) and variance increases (oscillation) as the context length grows.

%% file: content/n3_failure_modes.tex
\paragraph{Organization of the rest of the paper}
The rest of the paper formalizes how RoPE's intrinsic properties undermine the fundamental objectives of both \textit{position} and \textit{token} identification in long contexts. We begin with a theoretical analysis of a single attention head in \S\ref{sec:pos} and \S\ref{sec:tok}, where four specific failure modes are identified. 
For each, we first present a theoretical result and then provide empirical verification.
Our empirical analysis probes an attention head from Llama 3.1-8B \citep{grattafiori2024llama3herdmodels}, with a 128K claimed context length. We choose this model because of its popularity, moderate size, and representative decoder-only architecture.\footnote{Although Llama 3.1 uses RoPE scaling, we show in \cref{subsec:rope_scaling} that the analysis for standard RoPE still applies.} We illustrate the failure modes with a long context of mostly irrelevant text containing three relevant sentences: ``Alice has a cat,'' ``Bob has a dog,'' and ``What pet does Alice keep?'' We analyze the key tokens ``cat'' and ``dog'' and the query token ``pet''.  
We use the first head in the first layer as a case study, although our method applies to any head in any layer. See \cref{subsec:case_study_details} for implementation details.
In \S\ref{sec:exp}, we then turn to an empirical study of full multi-head, multi-layer language models.

\section{RoPE Fails to Distinguish Positions in Long Contexts}\label{sec:pos}

For the position identification objective, suppose that we are given a pair of fixed query and key tokens in an input of length $M$. The tokens may be placed at any position as long as the query token appears later. This means that the relative distance $m$ between the token pair satisfies $0\le m < M$.  

With recency bias, we expect that the key should have a high chance of receiving larger attention weights when it is closer than when the same key token is located farther away (i.e. $S(m_1)>S(m_2)$ where $m_1<m_2$). We identify two failure modes that violate this expected behavior and explain why they can be problematic.

\begin{figure*}[tbhp]
\centering
    \begin{subfigure}[t]{0.355\textwidth}
    \centering
        \includegraphics[width=\textwidth]{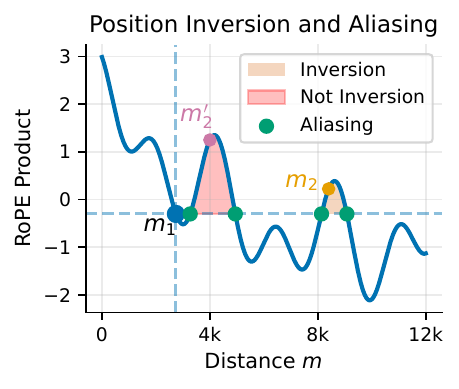}
        \caption{Position inversion occurs if $S({\color{orange}m_2})>S(m_1)$ despite $ {\color{orange}m_2} \gg m_1$. Local oscillations ({$m_1$, \color{purple}$m_2'$}) are \textit{not} considered.  }\label{fig:cosine_sum_with_three_marked_positions}
        \end{subfigure}
        ~
    \begin{subfigure}[t]{0.32\textwidth}
            \centering
        \includegraphics[width=\textwidth]{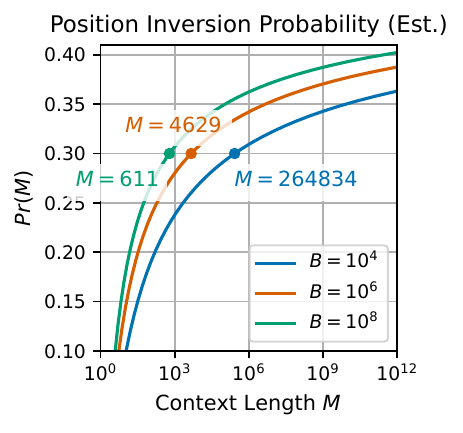}
        \caption{Probability estimation (lowerbound) of position inversion vs. context length. }
        \label{fig:position_inversion}
    
    \end{subfigure}
    ~
        \begin{subfigure}[t]{0.27\textwidth}
        \centering
        \includegraphics[width=\textwidth]{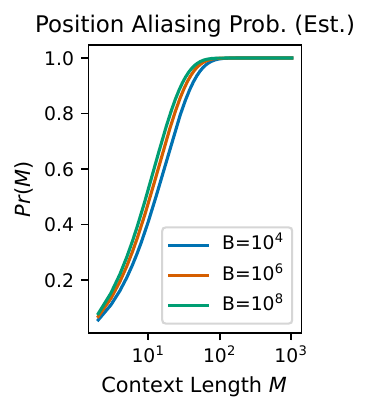}
        \caption{Probability of position aliasing at a random distance vs. context length.}
        \label{fig:position_aliasing_a}
    \end{subfigure}
    \centering
    \caption{ Illustrations (a) for position inversion and aliasing, with corresponding probability estimations under different RoPE \textbf{B}ases (b, c), $h=64$. 1k = 1,000.}
\end{figure*}

\subsection{Failure Mode 1: Position Inversion}\label{sec:position_inversion}

\textbf{Position inversion} is a reversal of RoPE’s locality inductive bias: 
given the query, moving the key to a  \emph{substantially farther} position \emph{increases} the attention score.
We focus on distant pairs drawn from opposite halves of the context, since such inversions are more detrimental than those among nearby tokens. We identify position inversions when $S(m_1)<S(m_2), m_1<M/2\le m_2$. 
See \cref{fig:cosine_sum_with_three_marked_positions} for an illustrative example.

\begin{thm}
\label{thm:positional_inversion}
The probability lowerbound of position inversion increases with context length $M$ and RoPE base $B$. The probability approaches $1/2$ as $\log M\log B\to \infty$.  
\end{thm}
Theorem \ref{thm:positional_inversion}  follows directly from treating the RoPE product as a normal random variable, as discussed in Remark~\ref{rmk:normal}. See \cref{subsec:prob-inv-pos} for the formal statement and proof.

Theorem~\ref{thm:positional_inversion} states that, given a query, moving the exact key token from a closer position $m_1\in[0,M/2)$ to a substantially farther position $m_2\in[M/2,M)$ can increase its attention score with probability approaching that of a coin flip. 
This is problematic because, as the context length and RoPE base grow, attention becomes nearly arbitrary in its preference between nearby vs. farther positions, making its behaviors less predictable. 
This unpredictability may prevent the model from identifying a reliable positional pattern.

In practice, the probability of position inversion can exceed $0.3$ even at short context lengths, as shown in \cref{fig:position_inversion}. Following a convention used in the Turing Test~\citep{turing1950computing}, we assume that this rate is already high enough to signal substantial positional ambiguity.

\begin{wrapfigure}[13]{r}{.65\textwidth}
\vspace{-.5cm}
\begin{subfigure}[t]{0.49\linewidth}
        \centering
    \includegraphics[width=0.98\linewidth]{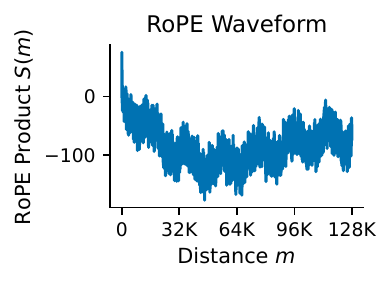}
    \caption{The RoPE decay happens only within the initial $\sim$ 50K tokens.  }
    \label{fig:position_inversion_example_rope_product_n3}
\end{subfigure}
~
\begin{subfigure}[t]{0.49\linewidth}
        \centering
    \includegraphics[width=0.98\linewidth]{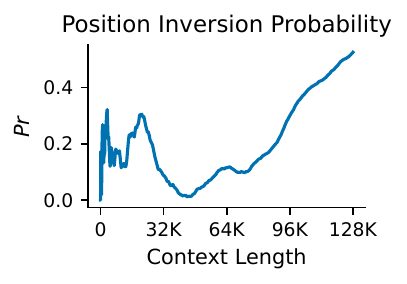}
    \caption{Position inversion probability approaches 0.5 as length increases. }
    \label{fig:position_inversion_example_prob_n3}
\end{subfigure}

\caption{Position inversion for key ``cat'' and query ``pet''. Llama 3.1-8B, Layer 0, Head 0. Here 1K = 1,024.}

\end{wrapfigure}

\paragraph{Empirical verification}
\label{subsec:case_study_position_inversion}
As shown in \cref{fig:position_inversion_example_rope_product_n3}, for the query token ``\texttt{pet}'', moving the key token ``\texttt{cat}'' across the advertised 128K context length of Llama 3.1-8B causes the attention score, i.e., the RoPE product, to reach a minimum at $m\approx 50\text{K}$. Beyond this point, the RoPE product exhibits an overall upward trend with oscillations, indicating position inversion.

\cref{fig:position_inversion_example_prob_n3} shows the corresponding probability of position inversion. Within just a few thousand tokens, this probability increases to nearly $0.3$; once $m>50\text{K}$, it continues to increase towards $0.5$. Note again that we consider only pairs $(m_1,m_2)$ where $m_1$ and $m_2$ lie in opposite halves of the full context. These inversions indicate that the model can fail to properly compare a nearby token with a substantially farther one.

\subsection{Failure Mode 2:  Position Aliasing}\label{sec:position_aliasing}

\textbf{Position aliasing} occurs when modifying the distance between query and key does not change the attention score at all.
Position aliasing can be seen as a complete failure to distinguish two different positions.
\cref{fig:cosine_sum_with_three_marked_positions} provides an illustration. An \textit{aliasing pair} refers to two distances with the same attention score.

\begin{thm}
\label{thm:position_aliasing}
The probability that a random distance admits an aliasing pair  converges to $1$ exponentially fast as the context length $M$ increases. Moreover, the total number of aliasing pairs increases with both the context length $M$ and the RoPE base $B$.
\end{thm}
The intuition behind Theorem~\ref{thm:position_aliasing} is that the difference between the RoPE products at two independent positions can be modeled as a zero-mean normal random variable. This allows us to estimate how often its absolute value falls below the datatype resolution used for the RoPE product. See \S\ref{subsec:position_aliasing_formal} for the formal statement and proof. See \cref{fig:position_aliasing_a} for the probability estimation. 

Theorem~\ref{thm:position_aliasing} states that position aliasing is inevitable with increased context lengths.
In practice, the issue can be amplified by limited numerical precision: it occurs when the difference between two RoPE products falls below the resolution limit of the data type.
Even when the attention scores for two distances are not exactly identical under higher precision, very small differences may be lost due to limited numerical precision.

\begin{figure*}[tbhp]
    
    \centering
    \begin{subfigure}[t]{0.32\textwidth}
        \centering
        \includegraphics[width=1.04\textwidth]{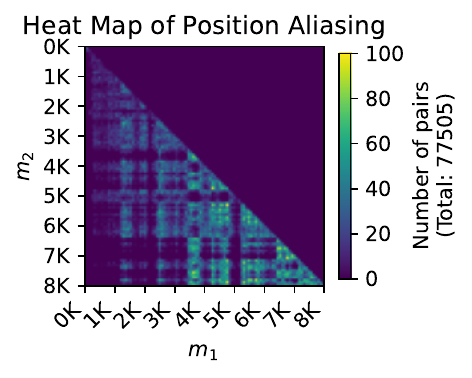}
        \caption{Position aliasing pairs for key ``cat'' and query ``pet''. 
        }
        \label{fig:position_aliasing_example_cat_n3}
    \end{subfigure}
    ~ 
    \begin{subfigure}[t]{0.32\textwidth}
        \centering
        \includegraphics[width=1.02\textwidth]{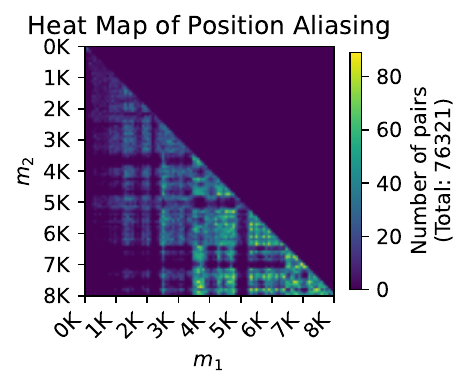}
        \caption{Position aliasing pairs for key ``dog'' and query ``pet''.}
        \label{fig:position_aliasing_example_dog_n3}
    \end{subfigure} 
    ~ 
    \begin{subfigure}[t]{0.32\textwidth}
        \centering
        \includegraphics[width=1.04\textwidth]{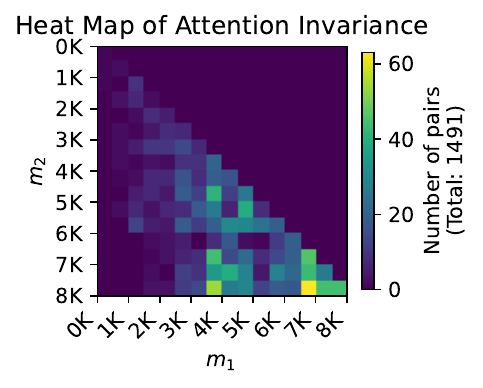}
        \caption{Attention invariance pairs for ``cat'', ``dog'' and ``pet''. 
        }
        \label{fig:attention_invariance_example_bf16_n3}
    \end{subfigure}
    \caption{Heat maps of position aliasing and attention invariance pairs under BF16, showing the ubiquity of position aliasing. Llama 3.1-8B, Layer 0 Head 0. Pairs are grouped into a total of $200\times 200$ bins for position aliasing, and $16\times 16$ bins for attention invariance. 1K = 1,024.}
    \label{fig:position_aliasing_examples_n3}
\end{figure*}

\paragraph{Empirical verification}
As shown in \cref{fig:position_aliasing_example_cat_n3,fig:position_aliasing_example_dog_n3}, under an 8K context length and commonly-used BF16 precision, almost every distance $m_1$ is involved in at least one aliasing pair $(m_1,m_2)$,
and there are already more than 75k aliasing pairs, the density of which increasing with the context length. This empirically confirms Theorem~\ref{thm:position_aliasing}
and suggests that position aliasing is a common issue even at relatively short context lengths.

\paragraph{Attention invariance caused by position aliasing} 
Position aliasing implies a specific failure mode: given a query $\mathbf{q}$ and two keys $\mathbf{k}_1$ and $\mathbf{k}_2$ at aliasing positions, swapping $\mathbf{k}_1$ and $\mathbf{k}_2$ does not change the attention output at all, as illustrated in \cref{fig:att_inv_ill}. \cref{fig:attention_invariance_example_bf16_n3} empirically verifies this failure mode, showing 1,491 such invariance cases even within an 8K context length. This further demonstrates that position aliasing can be damaging even at short context lengths.

\section{RoPE Fails to Distinguish Tokens in Long Contexts}\label{sec:tok}

For token identification, we may apply a similar analysis. 
 Let $\mathbf{q}$ be a query vector and let $\mathbf{k}_1$ and $\mathbf{k}_2$ be two key vectors. We consider the relative distances between the query and keys alone, but not a specific input context. Let $S_1(m)$ denote the RoPE product between $\mathbf{q}$ and $\mathbf{k}_1$ at distance $m$, and let $S_2(m)$ denote the corresponding RoPE product between $\mathbf{q}$ and $\mathbf{k}_2$. Assume that at $m=0$, where RoPE effectively has no effect, the first key is more relevant, i.e. $S_1(0)>S_2(0)$. 
Intuitively, 
 this relevance ordering should be preserved when both keys are placed at a new relative distance $m$, i.e. $S_1(m)>S_2(m)$. We identify the following violations.  

\subsection{Failure Mode 3: Token Inversion}\label{sec:token_inversion}

\term{Token inversion}
occurs when the relevance ordering of the two keys is reversed at distance $m$, i.e. $S_1(m) - S_2(m) < 0$ despite $S_1(0) > S_2(0)$ (See \cref{fig:token_aliasing_illust}).  
\begin{thm}
\label{thm:token_inversion}
The probability lower bound for token inversion increases with the context length $M$, approaching $1/2$ as $M$ approaches the natural context limit $\Theta(B)$. In contrast, the lower bound decreases with the RoPE base $B$.
\end{thm}
See \S\ref{subsec:prob-inv-imp} for the formal statement and proof.

Theorem~\ref{thm:token_inversion} states that RoPE can reverse the original ordering between two keys at some nonzero relative distance. Similarly to position inversion (\S\ref{sec:position_inversion}), the main problem with token inversion is its unpredictability: it can occur with probability approaching that of a coin flip. Suppose that $S_1(m)>S_2(m)$ for some values of $m$ but $S_1(m)<S_2(m)$ for others, with comparable frequencies; then it becomes unclear whether the model can reliably distinguish the two keys at those distances.

\begin{figure}[tbhp]
    \centering
    \begin{subfigure}[t]{0.49\textwidth}
        \centering
    \includegraphics[width=0.88\linewidth,trim={5pt 7pt 5pt 7pt},clip]{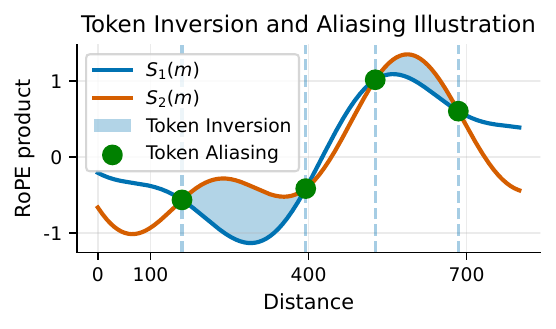}
    \caption{Token inversion: $S_1(0)>S_2(0)$ but \\$S_1(m)<S_2(m)$. Token aliasing: $S_1(m)=S_2(m)$.}
    \label{fig:token_aliasing_illust}
    \end{subfigure}
    ~
    \begin{subfigure}[t]{0.45\textwidth}
    \centering
    \includegraphics[width=0.74\linewidth,trim={5pt 7pt 5pt 7pt},clip]{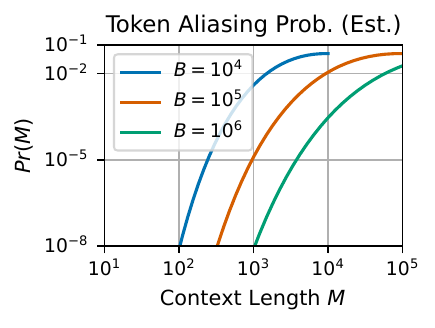}
    \caption{Typical token aliasing probabilities at different RoPE \textbf{B}ases, using BF16. 
    }
    \label{fig:token_aliasing}
    \end{subfigure}
    \caption{Illustration of token inversion and aliasing.}
\end{figure}

\begin{figure*}[tbhp]
    \centering
    \begin{subfigure}[t]{0.48\textwidth}
        \centering
        \includegraphics[width=0.99\textwidth]{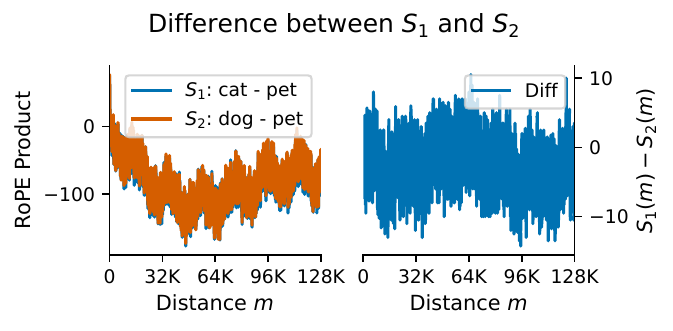}
        \caption{RoPE products for two key - query pairs (left) and their difference (right). }
        \label{fig:token_aliasing_rope_diff_n3}
    \end{subfigure}
    ~
    \begin{subfigure}[t]{0.47\textwidth}
        \centering
        \includegraphics[width=0.99\textwidth]{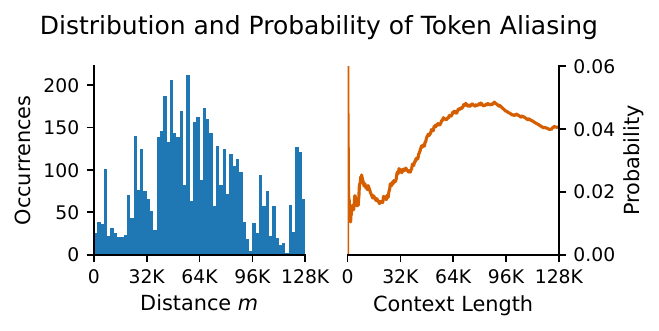}
        \caption{Distribution (left) and probability (right) of token aliasing. }
        \label{fig:token_aliasing_prob_bf16_n3}
    \end{subfigure} 
    
    \caption{Token aliasing probabilities under BF16, for query \texttt{pet}, and keys \texttt{cat} ($S_1$) and \texttt{dog} ($S_2$). Llama 3.1-8B, Layer 0 Head 0. Here, 1K = 1,024.}
    \label{fig:token_aliasing_examples_n3}
\end{figure*}

\paragraph{Empirical verification}
For the query token \texttt{pet}, we select a highly relevant key token, \texttt{cat}, and a less relevant key token, \texttt{number}. Let $S_1$ denote the RoPE product between \texttt{pet} and \texttt{cat}, and let $S_2$ denote the RoPE product between \texttt{pet} and \texttt{number}.
\cref{fig:token_inversion_prob} shows the difference $D=S_1-S_2$ and the probability curve of token inversion. Initially, $D>0$, as desired, indicating that \texttt{cat} receives a higher score than \texttt{number}. However, in fewer than 10 tokens, $D$ drops below zero and the relevance ordering between the two tokens is already reversed.
As $m$ increases, the probability of inversion exhibits an increasing lower bound, consistent with Theorem~\ref{thm:token_inversion}. When $m\ge 20K$, the  probability approaches 0.5; with an oscillating $D$, it becomes unpredictable whether \texttt{cat} or \texttt{number} receives the higher attention score.

\begin{wrapfigure}[12]{r}{.5\textwidth}
\vspace{-.5cm}
    \centering
        \includegraphics[width=\linewidth]{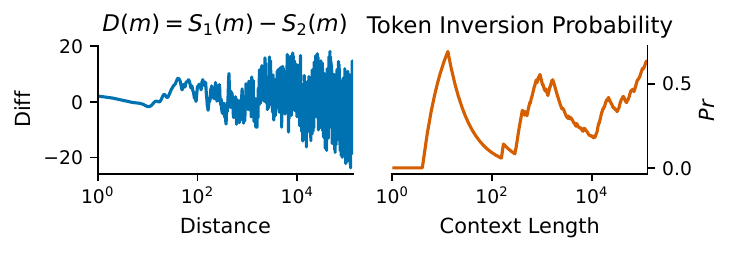}
        \caption{Token inversion for keys ``cat'', ``number'' and query ``pet''. Left: difference of RoPE product vs.  distance, where $S_1$ corresponds to ``cat'', and $S_2$ to ``number''. Right: Probability. Llama 3.1-8B, Layer 0 Head 0. }
        \label{fig:token_inversion_prob}
\end{wrapfigure}

\subsection{Failure Mode 4: Token Aliasing}
\label{subsec:token_aliasing}

At relative distance $m$, replacing $\mathbf{k}_1$ with a different key  $\mathbf{k}_2$ can leave the attention score unchanged, i.e., $S_1(m)=S_2(m)$. We refer to this phenomenon as \textbf{token aliasing}, which indicates that attention fails to distinguish between two different tokens at that position, as illustrated in \cref{fig:token_aliasing_illust}.

\begin{thm}
\label{thm:token_aliasing}
    The number of token aliasing positions increases with $M$ and decreases with $B$. For a sufficiently long context of length $M$, it is bounded by $\Theta(2^{-f}\sqrt{h}M)$, where $f$ is the explicit fraction bits of the data type used, and $h$ is the half hidden dimension.
\end{thm}
See \cref{subsec:token_aliasing_formal} for the formal statement and proof. 
Similar to position aliasing, token aliasing is amplified by limited numerical precision, such as BF16 where there are $f=7$ explicit fraction bits \citep{henry2019leveragingbfloat16artificialintelligence}.

\begin{table}[tbhp]
\caption{Summary of failure modes and how the chances of occurrence change with context length $M$ and RoPE base $B$. It is preferable to have {\color{ForestGreen}smaller} probabilities as opposed to {\color{red}larger} ones: for inversion this means better predictability, and for aliasing this means less ambiguity.}
    \centering
    \begin{tabular}{@{} l ccc @{}}

    \toprule
        Failure Mode & Indicator & $M\bm\uparrow$  & $B\bm\uparrow$ \\
        \midrule
        Position Inversion & $m_1<m_2$; $S(m_1)<S(m_2)$   & {\color{red}$\bm\uparrow$} & {\color{red}$\bm\uparrow$} \\ 
        Position Aliasing & $S(m_1)=S(m_2)$ &  {\color{red} $\bm\uparrow$} & {\color{red}$\bm\uparrow$} \\ 
        Token Inversion & $S_1(0)>S_2(0)$; $S_1(m)<S_2(m)$  & {\color{red}$\bm\uparrow$} & {\color{ForestGreen}$\bm\downarrow$} \\ 
        Token Aliasing & $S_1(m)=S_2(m)$ & {\color{red}$\bm\uparrow$} &  {\color{ForestGreen}$\bm\downarrow$} \\
        \bottomrule
    \end{tabular}
    
    \label{tab:failure_modes}
\end{table}

Token aliasing
can be  mitigated by increasing RoPE base 
(see \cref{fig:token_aliasing}); nevertheless, it is almost always present in long inputs. If $h=64$, using BF16, up to 5\% of the total positions exhibit token aliasing.
This means 1.6K aliasing positions in 32K tokens.

\paragraph{Empirical verification}
For query \texttt{pet} and two keys \texttt{cat} and \texttt{dog}, \cref{fig:token_aliasing_rope_diff_n3} shows the difference between the two corresponding RoPE products $D(m)=S_1(m)-S_2(m)$. 
Under BF16 precision, we count the aliasing positions where $D(m)=0$ and calculate the aliasing probability
(\cref{fig:token_aliasing_prob_bf16_n3}). According to Theorem \ref{thm:token_aliasing}, the probability should converge to a value of around 0.05, which matches the illustrated result. \cref{fig:token_aliasing_prob_bf16_n3} left shows that the frequency of token aliasing 
is negatively correlated with the RoPE products, which demonstrates how decay affects token aliasing.

We close \S\ref{sec:pos} and \S\ref{sec:tok} with \cref{tab:failure_modes} summarizing the four failure modes, along with the following takeaways:
\begin{AIbox}[find:finding]{Takeaways}{}
As the context length increases, all four failure modes become more likely, and RoPE-based attention becomes increasingly likely to fail at distinguishing both positions and tokens.
The choice of RoPE base $B$ trades off the failure modes:  
While increasing RoPE base can mitigate token inversion and token aliasing, it does not fully resolve them; moreover, it worsens position inversion and position aliasing.
In this sense, each attention head has a maximum effective context length. Beyond this limit, at least some of the four failure modes occur with sufficiently high frequency and compromise attention, regardless of how RoPE base is adjusted.
\end{AIbox}

%% file: content/211_evidence_real_model.tex
\section{How Do Multilayer, Multihead Transformer LLMs Fare?}
\label{sec:exp}

We have shown the four failure modes of a single attention head. But do multilayer, multihead Transformer LLMs overcome the limitations in practice? This section addresses this question empirically.

We conduct a controlled evaluation of six open RoPE-based long-context LLMs of different sizes, as shown in \cref{fig:llm}. We do not evaluate closed-source proprietary models, since their architectures and positional-embedding choices are not publicly known. Moreover, many recent long-context models are explicitly optimized for retrieval-based tasks, such as Needle-in-a-Haystack \citep{kamradt2023needleinahaystack}, which are essentially token-identification tasks. Our theoretical results predict that optimizing for distinguishing tokens inevitably trades off against distinguishing positions; therefore, we focus on the latter through controlled experiments.

\begin{figure}[tbhp]
    \centering
    \begin{subfigure}[t]{0.27\textwidth}
        \centering
        \includegraphics[width=0.99\textwidth]{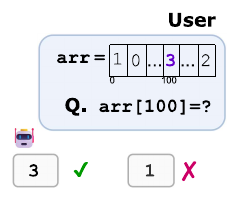}
        \caption{The indexing task.}
    \end{subfigure}
    ~
    \begin{subfigure}[t]{0.33\textwidth}
        \centering
        \includegraphics[width=1.06\textwidth]{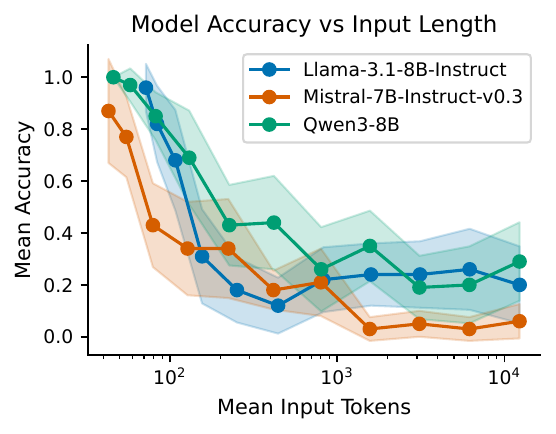}
        \caption{Small models ($\le$ 8B).}
        \label{fig:small_model_group}
    \end{subfigure}
    ~ 
    \begin{subfigure}[t]{0.33\textwidth}
        \centering
        \includegraphics[width=1.06\textwidth]{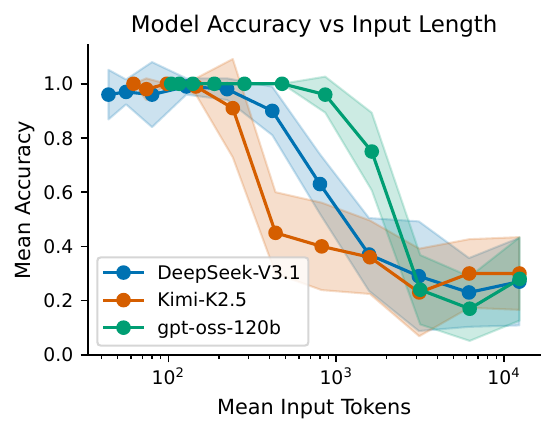}
        \caption{Large models ($\ge$ 100B).}
    \end{subfigure}

    \caption{(a) We ask models to extract the $k$-th element in an array consisting of only integers 0-3. (b, c)  Dots and shadows represent mean and standard deviation of accuracy. Selected models: \cite{grattafiori2024llama3herdmodels, mistralai_mistral_7b_instruct_v03_2024, yang2025qwen3technicalreport,deepseekai2025deepseekv3technicalreport, kimiteam2026kimik25visualagentic, openai2025gptoss120bgptoss20bmodel}}
    \label{fig:llm}
\end{figure}

\paragraph{The indexing task}

In our position identification task, the model is presented with a Python list \texttt{arr}, and is required to answer the value of the given index \texttt{arr[i]}.  
Each model receives the same input. For each list length, we test the models with 100 samples. The list only contains integers 0, 1, 2, and 3. This turns the task into a multiple-choice problem, allowing us to compare with the random-guess accuracy: if a model cannot identify the single target position (each element takes 1 token), it answers a random element with an accuracy of 0.25.

\paragraph{Results}
As shown in \cref{fig:llm}, all models start with near-perfect accuracy but quickly drop to a level close to random guessing. 
These real long-context models face serious position confusion with as short as several thousands of tokens, a cost they pay for optimizing token identification. 

See \cref{sec:exp_det} for more details on the experiment settings.

%% file: content/300_discussions.tex
\section{Conclusion and Discussion}
\label{sec:discussions}

Our theoretic analysis shows that RoPE intrinsically fails to distinguish both position and token identity in long inputs, and the selection of RoPE base only acts as a tradeoff between the two objectives. The same bottleneck persists in practical multi-head multi-layer models, as observed in the empirical verifications.
 
Our conclusion suggests that a robust positional mechanism should maintain the ability to distinguish positions and tokens, and context length cannot be effectively extended without addressing the two objectives. 
At the same time, we are encouraged by recent efforts that explore alternatives to direct length extension, including improved context management and alternative paradigms such as recursive or agentic language models. Although these approaches do not directly resolve the intrinsic limitations of RoPE identified in this work, they reflect a broader shift away from treating context extension alone as sufficient. We hope that this work encourages further research into fundamentally new positional mechanisms for long-context language models.

%% file: content/limitations.tex
\section*{Limitations}
\label{sec:limitations}
\paragraph{Ambiguity of the frequency threshold} It may be noticed in the illustrations of normal approximation of the RoPE product, that the real distribution is slightly skewed compared to our approximation. We provide only a rough threshold $\lambda(M)=\Theta(h\log_BM)$, following the common practice of dividing the entire frequency domain into two categories of high and low. Strictly speaking, the threshold frequencies where $n=\Theta(h\log_BM)$ should be considered separately: their rotations are around 1 complete circle. We mostly ignore these terms, since there are only $O(h/\log B)$ of them, usually not exceeding 10. We believe that ignoring them only affects the precision of approximation, not the fundamental conclusions.

\paragraph{Assumption of regular rotary amplitudes} We assume in this paper that the amplitudes of the RoPE products, $\{a_n\}$, are relatively uniform (no terms dominate). This is for simplicity in calculation. In real attention heads, we often observe several dimensions that have amplitudes significantly larger than others. While we did not discuss this in the paper, these dominating terms reduce the effects of other components; this heuristically works like reducing the total number of dimensions, and simplifying the oscillation patterns. The effective context length limit shortens to the maximum wavelength with a non-negligible amplitude. Smaller amplitudes for high frequency terms reduce oscillation; this may affect position identification for close position pairs, which is discussed in \cite{liu2026rotarypositionalembeddingsphase}. In short, we assume that non-regular amplitudes are suboptimal, shortening the effective context length in the sense of the failure modes we discuss in this paper.

\paragraph{Real models} We do not theoretically analyze the multi-layer multi-head attention, or conduct experiments that categorize the failure modes, since real models are significantly more complicated, and one actual failure can be caused by mixed factors.  Our final experiment merely serves as a demonstration that real models, regardless of their size, still face either position or token confusion, which we analyzed for a single head, and must sacrifice position confusion for better token identification. We do not intend to propose any new benchmarks or metrics, or make the task any realistic.  While our experiments on real models do not directly prove that the failures are actually caused by the RoPE confusions, they indicate that the redundancy provided by multiple heads and layers may have limited protection.  In \cref{sec:real_model_discussions}, we provide discussions that may serve as a preliminary insight for future investigations on how the failures accumulate with heads and propagate across layers. 

\paragraph{RoPE scaling} While we show in \cref{subsec:rope_scaling} that certain RoPE scaling methods may be reduced to standard RoPE, we do not provide a more detailed analysis for any specific variant. However, we use Llama-3.1 in our case study, which uses RoPE scaling, to illustrate that the scaling does not fundamentally resolve the problem, if it helps at all.

%% file: content/acknowledgement.tex
\section*{Acknowledgement}

This work was supported by 
 a grant from Coefficient Giving, an Amazon AICE Award, gift funding from AI2,
 and by
Laboratory Directed Research and Development (LDRD) funding from
Argonne National Laboratory, provided by the Director, Office of Science, of the U.S. Department
of Energy under Contract No. DE-AC02-06CH11357. The work used resources of the Argonne
Leadership Computing Facility, a DOE Office of Science User Facility supported under Contract
DE-AC02-06CH11357.  EAH acknowledges support from National Science Foundation (NSF) awards
OAC-2514142 and OAC-2209892 . PH was supported by Deutsche Forschungsgemeinschaft (DFG) Grant No. SFB-TRR 358/1 2023 — 491392403. This research also used the Delta advanced computing and data resources, which
is supported by the National Science Foundation (award OAC 2005572) and the State of Illinois.
Delta is a joint effort of the University of Illinois Urbana-Champaign and its National Center for
Supercomputing Applications. This research used the DeltaAI advanced computing and data resource,
which is supported by the National Science Foundation (award OAC 2320345) and the State of
Illinois. DeltaAI is a joint effort of the University of Illinois Urbana-Champaign and its National
Center for Supercomputing Applications.

%% file: content/400_appendix.tex
\appendix

\section{Rotary Positional Embedding}

Let $X=(x_0, x_1,\ldots,x_{M-1})$ be the input word embeddings, where $M$ is the number of input tokens, and $x_i\in \mathbb{R}^d$ be the word embedding of token $i$. The word embeddings are then transformed into query, key and value representations, the transformation typically being linear, where the positional information is embedded.

RoPE applies the same multiplication-based positional embedding to the query vector $\mathbf{q}$ and key vector $\mathbf{k}$:

\begin{align*}
\hat q_i=R_{\Theta,i}^d q_i=R_{\Theta,i}^dW_Qx_i,\\
\hat k_j=R_{\Theta,j}^d k_j=R_{\Theta,j}^dW_Kx_j,
\end{align*}

where

$$
R_{\Theta,m}^d=\begin{pmatrix}
\cos m\theta^0 & -\sin m\theta^0 & 0 & 0 & \ldots & 0 & 0 \\
\sin m\theta^0 & \cos m\theta^0 & 0 & 0 & \ldots & 0 & 0 \\
0 & 0 & \cos m\theta^1 & -\sin m\theta^1 & \ldots & 0 & 0 \\
0 & 0 & \sin m\theta^1 & \cos m\theta^1 & \ldots & 0 & 0 \\
\vdots & \vdots & \vdots & \vdots & \ddots & \vdots & \vdots \\
0 & 0 & 0 & 0 & \ldots  & \cos m\theta^{d/2-1}& -\sin m\theta^{d/2-1} \\
0 & 0 & 0 & 0 & \ldots   & \sin m\theta^{d/2-1}& \cos m\theta^{d/2-1}\\
\end{pmatrix}.
$$

The attention mechanism features the ``attention score'', defined as

$$
A=\text{softmax}\left( \frac{\hat Q^\intercal\hat K}{\sqrt{d}}\right),
$$

where $\hat Q=(\hat q_0, \hat q_1, \ldots,\hat q_{M-1})$ and $\hat K=(\hat k_0,\hat k_1,\ldots,\hat k_{M-1})$.

The attention score $a_{i,j}$ is the normalization of the inner product between $\hat q_i$ and $\hat k_j$. For the sake of simplicity, we call this inner product the ``RoPE product''. RoPE is designed such that only the relative distance, $i-j$, matters in RoPE product: 

$$
\hat q^\intercal_i\hat k_j= (R_{\Theta,i}^dq_i)^\intercal(R_{\Theta,j}^dk_j)=q_iR_{\Theta,i-j}^dk_j.
$$

\paragraph{Conventions}
We use the following conventions throughout this paper:

\begin{itemize}
    \item All indices start from 0;
    \item $d$ --- hidden dimension of an attention head;
    \item $h\equiv d/2$ --- half dimension;
    \item $B$ --- the RoPE base (also called ``the RoPE theta'' in some literature);
    \item $\theta\equiv B^{-1/h}$ --- the basic RoPE frequency (not ``the RoPE theta''); 
    \item $m$ --- the distance between the query token, $\text{tok}_{\text{idx}_\text{query}}$, and the key token, $\text{tok}_{\text{idx}_\text{key}}$, defined as $m=\text{idx}_\text{query} - \text{idx}_\text{key}$. For causal models, $m\ge0$; 
    \item $\mathbf{q},\mathbf{k}$ --- the query and key vectors;
    \item $S_{\mathbf{q},\mathbf{k}}(m)$ --- the RoPE product of the given query and key vectors, w.r.t. their distance $m$. By default, when referring to $\mathbf{q}$ and $\mathbf{k}$, we use the abbreviated $S(m)$;
        \item $\mathbf{a}, \bm{\phi}$ --- the amplitudes and phases of the RoPE product, defined as
    \begin{align*}
a_n=&a_n(\mathbf{q}, \mathbf{k})=\sqrt{(q_{2n}^2+q_{2n+1}^2)(k_{2n}^2+k_{2n+1}^2)}, \\
\phi_n=&\phi_n(\mathbf{q}, \mathbf{k})=\text{atan2}(q_{2n}k_{2n+1}-q_{2n+1}k_{2n},q_{2n}k_{2n}+q_{2n+1}k_{2n+1});
\end{align*} 
    \item $\lambda(M)$ --- the threshold dimension $\lambda(M)=\Theta(h\log_B M)$.
\end{itemize}

\begin{table}[tbhp]
\caption{A summary of notations we use. }
    \centering
    \begin{tabular}{cc}
    \toprule
        Notation & Description \\
        \midrule
        $d$ & Hidden dimension for an attention head \\
        $h$ & Number of rotary components $h=d/2$ \\
        $B$ & RoPE base, the most important RoPE hyperparameter \\
        $\theta$  & Base frequency $\theta=B^{-1/h}$ \\
        $\mathbf{q}, \mathbf{k} \in \mathbb{R}^d$ & query and key vectors \\
        $M$ & Context length limit \\
        $S(m)$ & RoPE product of $\mathbf{q}$ and $\mathbf{k}$ at distance $m$ \\
        $a_n, \phi_n$ & Amplitude and phase of the $n$-th rotary component  \\
        $\lambda(M)$ & Threshold frequency index \\
        \bottomrule
    \end{tabular}
    
    \label{tab:placeholder}
\end{table}

\paragraph{The High and Low Frequencies of RoPE Product}

We use the threshold value $\lambda =\lambda(M)=\Theta(h\log_B M)$ to separate high and low frequencies: the terms with smaller indices, $n\ll \lambda(M)$, are considered high frequency terms, and the others are low frequency terms. This is not a strict division; generally, low frequency terms should rotate no more than a complete circle within the context length limit.  

The high frequency components correspond to the oscillation effect. Across the distance $m$, these rapid rotations create a distinction in the RoPE products of close position pairs. 

The low frequency components cause the decaying effect (the recency bias feature). The low frequency terms rotate slowly, and are generally within the decaying interval. More specifically, when $n\gg \lambda(M)$, we have $m\theta^n+\phi_n\in [0, \pi]$, so $\cos(m\theta^n+\phi_n)$ decreases with $m$ on $m\in[0, M)$. This decaying effect identifies distant positions.   \cref{fig:rope_product_vs_distance_b} illustrates both the oscillation and the decaying effect. 

The low frequency components also help preserve token relevance. For a low frequency term $n\gg \lambda(M)$ with very little rotation, $m\theta^n+\phi_n\approx \phi_n$, and the cosine term changes little from its initial value. It has been shown \citep{miranda2024round, jonasson2025rotaryoffsetfeatureslarge} that for query-key pairs that reflect longer text dependency, the low frequency terms tend to have higher amplitudes $a_n$, so that the RoPE product is less affected by the distance. 

\paragraph{The Natural Context Length Limit} \label{par:appendix_base_wavelength} If the context length $M$ is too large, even the lowest frequency term begins to oscillate. From a signal processing perspective, the lowest frequency term $n=h-1$ determines the fundamental frequency $\theta^{h-1}=B^{-\frac{h-1}{h}}\approx B^{-1}$, with a fundamental wavelength of $2\pi B$\footnote{Assuming that the lowest frequency amplitude $a_{h-1}>0$. 
}. If $M\ge 2\pi B$, the fundamental frequency term loses its uniqueness and the positional embedding becomes ambiguous \citep{liu2026rotarypositionalembeddingsphase}. Therefore, given a RoPE base $B$, $M < 2\pi B$ is a natural upperbound for context length $M$. If we expect RoPE to maintain recency bias, then the lowest frequencies must be in their decreasing phase, lowering this bound even further to around $\pi B$. In this paper, we use a rough estimation $\Theta(B)$ as the natural context length limit; by default, $M$ does not exceed this limit.

\paragraph{Decay} \label{par:decay} It is worth noticing that for the actual RoPE product $S(m)$, the decay is not universal. First, the decay is the collective effect of multiple low-frequency rotations, so it occurs only within an interval, after which even the lowest frequency starts oscillating. Second, it is possible that the RoPE product does not decay from the beginning; the initial decay only occurs when the low frequency phases $\phi_n$ are mostly small $(<\pi/2)$. In fact, for any distance $m$, we may construct a phase vector $\bm{\phi}$ so that $S(m)$ reaches its maxima at $m$  \citep{miranda2024round}
\footnote{ A more accurate statement on decay should be: if $\widetilde{S}\sim N(\mu_M, \sigma_M^2)$, then the upperbounds of the 99.7\% probability threshold $\mu_M+3\sigma_M$ decrease with  $M \in [0, B^{h_1/h})$, where $h_1$ is the largest index such that $a_{h_1}/a_{\max}\gg 0$, i.e. the lowest frequency component with a non-negligible amplitude.}. 
It is not uncommon in reality that $S(m)$ reaches its peak before starting the decay. However, as we show in \cref{sec:pos}, the decay is a preferred behavior: a lack of decay indicates that there is no global monotonicity, leading to a failure to identify distant positions (position inversion).

\paragraph{Non-Uniform Rotary Amplitudes} We assume that the magnitudes $\{a_n\}$ across frequency terms are relatively uniform. When, as seen in real models, the low frequency terms instead have a negligible magnitude compared to high frequency ones, the actual base wavelength shrinks to match with the first term with a substantial magnitude. This in effect equals applying a smaller $B$: it shortens the natural context length limit, and causes a shorter decay interval. 

\section{RoPE Product Can Be Seen as a Normal Variable}
\label{sec:rope-sum-as-normal}

Consider the high frequency parts where $n<  \lambda(m)$. For any integer $0<A\le h$, define the partial RoPE product as 

$$
S_{n<A}(m)=\sum_{n<A} a_n\cos(m\theta^n+\phi_n).$$

 We shall show that if $m$ is uniformly chosen on $[A, M)$ where $M-A$ is large, then $S_{n<\lambda(M)}(m)$ heuristically behaves like a normal variable. Then, by estimating the low dimension terms as $S_{n\ge\lambda(M)}(m)\approx \sum_{n\ge\lambda(M)}a_n\cos\phi_n$, we can estimate the distribution of the whole RoPE product.

Suppose that we are interested in the integer interval $I_A=[A, A+M)$, on which $m$ is a uniform random variable. On $I_A$, define the mean exponential sum

\begin{align*}
    G_M(\alpha, A)=&\frac{1}{M}\sum_{m=A}^{A+M-1} e^{i\alpha m}\\
    =&\frac{\sin(M\alpha/2)}{M\sin(\alpha/2)} e^{i\alpha(A+(M-1)/2)}.
\end{align*}

We have

$$
|G_M(\alpha, A)|\ll \min\left(1, \frac{1}{M\sin(\alpha/2)}\right).
$$

For higher frequencies, $M\gg \frac{1}{\sin(\alpha/2)}$, and $|G_M(\alpha, A)|\ll 2 / M\theta^n.$

\paragraph{The Moments: Expectation and Variance} The mean cosine sum, $\mu_M(A)$, can be defined as 

\begin{align*}
    \mu_M(A)=&\frac{1}{M}\sum_{m=A}^{A+M-1}S(m)\\
    =&\sum_{n=0}^{\lambda(m)-1} a_n \Re(e^{i\phi_n}G_M(\theta^n, A)).
\end{align*}

Let us first calculate the variances of individual frequency terms. For the weighted cosine value at the $n$-th dimension pair, the term $\Psi_n=a_n\cos(m\theta^n+\phi_n)$ is a random variable, with the following results:

\begin{align*}
    E_{m}[\Psi_n]=&a_n\Re(e^{i\phi_n}G_M(\theta^n,A))\approx 2a_n/M\theta^n\to 0,\\
    E_{m}[\Psi_n^2]=&\frac{a_n^2}{2}\left(1+\Re(e^{2i\phi_n}G_M(2\theta^n,A))\right)\approx \frac{a_n^2}{2}\left(1+\frac{1}{4M\theta^n}\right)\to \frac{a_n^2}{2}.
\end{align*}

Next, let us estimate the cross variance terms. Let $n, p$ be two frequency indices such that $n < p$. Using $\cos a\cos b=\frac{1}{2}\cos(a+b)+\frac{1}{2}\cos(a-b)$, and $\theta^n-\theta^p \asymp \theta^n+\theta^p\asymp\theta^n,$

\begin{align*}
    |E_m[\Psi_n\Psi_p]|\ll\frac{a_na_p }{M(1-\theta)\theta^n}\sim \frac{a_na_p}{M\theta^n}(1/2+\frac{h}{\log B}).
\end{align*}

When $M\theta^{n}\gg 1$ ($h/\log B$ to be exact), the only significant terms related to the variance are the variants of individual cosine terms, and we may ignore the covariances, with $O(h/\log B)$ exceptions of size $O(a_na_ph/\log B)$.

\begin{align*}
    Var(S_{n<\lambda(M)})=&\sum_{n<\lambda(M)}Var(\Psi_n)+\sum_{n<\lambda(M)}\sum_{n<p<\lambda(M)}Cov(\Psi_n, \Psi_p) \\
    \asymp & \sum_n a_n^2/2+o(\frac{\sum_{n,p} a_na_p}{\log^2 B}).
\end{align*}

The second (covariance) term depends on $\frac{1}{1-\theta}=\Theta(h/\log B+1/2).$ We shall proceed to show that this term may be neglected. In fact, if the numbers 

\begin{align}
\frac{1}{2\pi}, \frac{\theta}{2\pi}, \ldots, \frac{\theta^{\lambda(M)}}{2\pi}
\label{eq:dependency}
\end{align}

are linearly independent over $\mathbb{Q}$ (i.e. for any $\mathbf{k}\in \mathbb{Q}^{\lambda(M)+1}$, $\sum_nk_n\theta^n/2\pi=0$ iff $\mathbf{k}=\mathbf{0}$), then according to the multi-dimensional generalization of Weyl's Criterion, the vector sequence $\{\mathbf{v}_m\}$ is equidistributed modulo $1$, where

$$
\mathbf{v}_m=(\frac{m}{2\pi}, \frac{m\theta}{2\pi}, \ldots, \frac{m\theta^{\lambda(M)}}{2\pi}).
$$

The independence of sequence \ref{eq:dependency} is equivalent to that of the sequence

$$
1, \theta, \ldots, \theta^{\lambda(M)}.
$$

Indeed, this sequence is linearly dependent over $\mathbb{Q}$ if and only if $\theta$ is a root of a rational polynomial of degree at most $\lambda(M)<h$. 

\paragraph{Pseudo Independence}
Since $\theta=B^{-1/h}$, or equivalently $B\theta^h-1=0$, where $h$ is a positive integer, $\theta$ may be a root of such a polynomial only if $B$ is a perfect power, i.e. $\exists g,c, \alpha,\beta, \gamma_i \in \mathbb{N}$ s.t. $B=c^\beta$, $c=\prod_i p_i ^{\gamma_i}$, $g=\beta\cdot \gcd_i\{\gamma_i\}$ and $\text{gcd}(g, h)>1$. Actually, this can be easily constructed: an example is $B=65536=2^{16}, h=32$, where $\theta=1/\sqrt{2}$ is a root with degree 2. However, even if $\theta^k$ is rational for some $k=h/\text{gcd}(g, h)$, this only means that the sequence can be divided into $O(h/k)=O(\text{gcd}(g, h))$ independent groups of size $O(k)$. One bound is that $k\ge 2$ whenever $B<2^{h}$.  For most cases where $B$ is not a perfect power, $\gcd(g,h)=1, k=h$. Therefore, we can heuristically view $X_j=m\theta^j$ as (largely) independent random variables that follow uniform distribution on $[0, 2\pi)$, with negligible covariances between any two terms. 

\paragraph{What if $k$ is small?} We only need to study the covariances of the dependent pairs. Let $n,p$ be the indices of a dependent pair where $n<p, n<\lambda(M), B=b^{h/k}, \theta=b^{-1/k}, b\ge 2$, $k=2$ (the smallest $k$ possible). Then $p= ckn$ for some integer $c\ge 1$.

\begin{align*}
    \Psi_n\Psi_p=&\frac{1}{2}a_na_p(\cos(m\theta^n+m\theta^p+\phi_n+\phi_p)+\cos(m\theta^n-m\theta^p+\phi_n-\phi_p))
\end{align*}

where $m\theta^n+m\theta^p=m\theta^n(1+\theta^{(ck-1)n})\ge m\theta^n,$
$m\theta^n-m\theta^p=m\theta^n(1-\theta^{(ck-1)n})\ge m\theta^n(1-\theta^n),$

so

$$
E_m[\Psi_n\Psi_p] \ll \frac{a_na_p}{M \theta^n(1-\theta)}=\frac{a_na_p}{M\theta^n(1-b^{-1/2})} \le \frac{a_na_p}{(1-1/\sqrt{2})M\theta^n} = o( a_na_p).
$$

If $k$ is small, the covariances between dependent terms become much smaller than the variance terms ($\Theta(a_n^2)$) and may be ignored.

\paragraph{Normal Approximation} Since $\{a_n\}$ is assumed with no dominating terms, we can apply Lindenberg's CLT and approximate the sum $S_{n<\lambda(M)}$ as a normal distributed random variable with zero mean. 

The estimated distribution of the full RoPE sum is therefore:

\begin{align}
    \widetilde{S}_M\sim N(\mu_M, \sigma^2_M),
\end{align}

where

\begin{align}
    \mu \approx & \sum_{n=\lambda(M)}^{h-1} a_n\cos\phi_n \le \sum_{n=\lambda(M)}^{h-1} a_n, \label{eq:normal_mu} \\
    \sigma \approx & \sqrt{\sum_{n=0}^{\lambda(M)-1} \frac{a_n^2}{2}}. \label{eq:normal_sigma} 
\end{align}

\paragraph{Error Estimation} Since CLT is used to derive the normal approximation, it is good when $\lambda(M)$ is large. Following the Berry-Esseen Theorem \citep{esseen1942liapunov}, the distribution approximation error is on the order of
$
O\left(\frac{\rho}{\sigma^3\sqrt{\lambda(M)}} \right),
$
where $\rho=E(|a_n\cos x_n|^3)=O(a_n^3).$ Therefore, the error is 
$$O\left(\frac{1}{\sqrt{\lambda(M)}}\right).$$

Empirically, the practical convergence is usually better, and $\lambda(M)>20$ is usually good enough. For $h=64$, this means that the estimation is good when $M>\sqrt[3]{B}$.

\subsection{RoPE Scaling}
\label{subsec:rope_scaling}
For any variant of RoPE, we need to analyze the linear independence of its frequencies (angular frequencies $\omega_n$ over $2\pi$), i.e.

$$
\frac{\omega_0}{2\pi}, \frac{\omega_1}{2\pi}, \ldots,\frac{\omega_{\lambda(M)}}{2\pi}
$$

We study RoPE scalings where the $n$-th angular frequency is a polynomial of $\theta^n$, i.e.

$$
\omega_n=\sum_{p=0}^{E_n}\kappa_{n,p}\theta^{np}, \kappa_{n,E_n} > 0, E_n>0.
$$

This includes most RoPE scaling variants, such as the NTK scaling \citep{bloc97_2023_ntkaware_rope} used in Llama models.

If all $\kappa_{n,p}$ are rational, then such independence can be largely analyzed the same way: it can be secured if $\theta$ is not a root of any rational polynomial of $\lambda(M)\cdot \max \{E_n\}$. Again, this is rarely the case. If for some specific selections of $B$ andh $h$, $\theta^{k\max E_n}$ is rational for some $k>1$, then $1-\theta$ should be $O(1)$, and the resulting covariances of the dependent terms should also be negligible. 

The subsequent analyses are similar. The only major difference is that $\lambda(M)$ should be modified accordingly to satisfy $M\omega_{\lambda(M)}=\Theta(1)$. 

\section{The Failure Modes}
\label{sec:failure_modes}

\subsection{Position Inversion}
\label{subsec:prob-inv-pos}

\begin{definition}
\label{def:position_inversion}
For a pair of query and key vectors $\mathbf{q},\mathbf{k}$, context length limit $M$, uniformly select $m_1 \in [0, M/2)$ and $m_2\in [M/2, M)$. 
A \textbf{position inversion} occurs for the pair $(m_1, m_2)$ if and only if $S(m_1)<S(m_2).$ The probability of ``position inversion'' can be defined as 

\begin{align*}
    Pr(\text{inversion}|\theta, M, \mathbf{q},\mathbf{k})=Pr(S_{\mathbf{q},\mathbf{k},\theta}(m_1)<S_{\mathbf{q},\mathbf{k},\theta}(m_2)).
\end{align*}
\end{definition}

    The difference between the two RoPE products,

\begin{align*}
    D=& S(m_1)-S(m_2),
\end{align*}

may be separately viewed as the difference of two independent normal variables following \cref{sec:rope-sum-as-normal}:

$$
\widetilde{D}= \widetilde{S}_{M/2}-\widetilde{S}_{M}\sim N(\mu_1-\mu_2, \sigma_1^2+\sigma_2^2),
$$

where

\begin{align*}
    \mu_1-\mu_2=&\sum_{n=\lambda(M/2)}^{\lambda(M)} a_n\cos\phi_n \le \sum_{n=\lambda(M/2)}^{\lambda(M)} a_n,\\
\sigma_1^2+\sigma_2^2=&1/2\sum_{n=\lambda(M/2)}^{\lambda(M)} a_n^2+\sum_{n=0}^{\lambda(M/2)} a_n^2.
\end{align*}

and the probability of position inversion should be

\begin{align*}
    Pr(\widetilde{D}<0)=\Phi\left(-\frac{\mu_1-\mu_2}{\sqrt{\sigma_1^2+\sigma_2^2}}\right),
\end{align*}

where $\Phi$ is the cumulative distribution function of the standard normal distribution.

To more accurately estimate the probability, we need the actual values of $a_n$ and $\phi_n$, which are dependent on $\mathbf{q},\mathbf{k}$. However, if we assume that $\{a_n\}$ are regular enough, then we have the following estimation:

\begin{thm}
\label{thm:positional_inversion_formal}
If $\sum_n{a_n}/h \approx \sqrt{\sum_n{a_n^2}/h} \approx a_{\max}>0$, then
\begin{align}
    Pr(\widetilde{S}_{M/2}-\widetilde{S}_{M}<0) \ge \Phi\left(- \log 2\sqrt\frac{{h}}{{\log B(\log M-0.5\log 2)}}\right). \label{eq:lowerbound_pr_pos_inv}
\end{align}
 
\end{thm}

\begin{proof}

When $M$ is so large that $\lambda(M)=\lambda(M/2)=h$, then $\mu_1=\mu_2$. In this case, $\widetilde{D}\sim N(0, \sigma_1^2+\sigma_2^2)$, and the probability $Pr=0.5$. Since $- \log 2\sqrt\frac{{h}}{{\log B(\log M-0.5\log 2)}}<0$, the right hand side of \cref{eq:lowerbound_pr_pos_inv} is less than 0.5 and therefore less than the questioned probability.

When $M\le B$, we have $h\ge \lambda(M)>\lambda(M/2)$, 
 
\begin{align*}
    \frac{\mu_1-\mu_2}{\sqrt{\sigma_1^2+\sigma_2^2}}\lesssim & 
    \frac{\lambda(M)-\lambda(M/2)}{\sqrt{\lambda(M)/2+\lambda(M/2)/2}}. \\
    \le &  \frac{h\frac{\log 2}{\log B}}{\sqrt{h\frac{2\log M-\log 2}{2\log B}}} \\
    = & \log 2\sqrt\frac{{h}}{{\log B(\log M-0.5\log 2)}} .
\end{align*}

When $h = \lambda(M)>\lambda(M/2)$, we have

\begin{align*}
\frac{\mu_1-\mu_2}{\sqrt{\sigma_1^2+\sigma_2^2}}\lesssim & 
    \frac{h-\lambda(M/2)}{\sqrt{h/2+\lambda(M/2)/2}}.
\end{align*}

Let $x=\lambda(M/2)\in [h(1-\log_B 2), h]$, and $f(x)=\frac{h-x}{\sqrt{h/2+x/2}}$. Since $f$ is a continuous function that decreases with $x$, it reaches its maximum at $x=h(1-\log_B 2)$, that is when $M=B$. Then, $f(x)=\log 2\sqrt\frac{{h}}{{\log B(\log B-0.5\log 2)}}=\log 2\sqrt\frac{{h}}{{\log B(\log M-0.5\log 2)}}.$   

\end{proof}

This gives a lowerbound for the probability of position inversion. As $\log M$ or $\log B$ increases, $\frac{\mu_1-\mu_2}{\sqrt{\sigma_1^2+\sigma_2^2}}$ decreases, and the probability of position inversion increases. As shown in the proof, when $M \to \Theta(B)$, this probability goes to $1/2$. On the other hand, if $B$ increases, this probability also increases; if we apply a threshold probability $\alpha$, and constrain $Pr<\alpha$, then as $B$ increases, the maximum possible $M$ decreases. See \cref{tab:prob_pos_inversion_b_m}.

\begin{table}[hbtp]
    \caption{Smallest $M$ for selected $B$ such that the probability of position inversion is no less than 0.3, when $h=64$.}
    \centering
    \begin{tabular}{cccccc}
    \toprule
        $B$ & $10^4$ & $10^5$ & $10^6$ & $10^7$  & $10^8$\\
        \midrule
        $M$ & $2.65\times 10^5$ & $23361$ & $4630$ & $1457$ & $612$\\
        \bottomrule
    \end{tabular}

    \label{tab:prob_pos_inversion_b_m}
\end{table}

\subsection{Position Aliasing}
\label{subsec:position_aliasing_formal}
\begin{definition}
\label{def:position_aliasing}
    Let $S(m)$ be the RoPE product of $\mathbf{q},\mathbf{k}$ at distance $m$, and $\hat{S}(m)=\hat{S}_{\texttt{dtype}}(m)$ be the numerical result of $S(m)$ using a certain datatype $\texttt{dtype}$. If for $m_1\neq m_2$, $\hat{S}(m_1)=\hat{S}(m_2)$, then the pair $(m_1, m_2)$ exhibits \textbf{position aliasing}.
\end{definition}

This is not to be confused with the concept of ``aliasing''  in signal reconstruction, where a high frequency signal beyond the Nyquist Limit is mistaken as a lower one \citep{liu2026rotarypositionalembeddingsphase}. This type of aliasing happens when $M>\Theta(B)$, as pointed out in \cref{par:base_wavelength}. 

For any integer $M>0$ and error $\varepsilon > 0$, which is related to the datatype resolution, if we independently take two random distances $m_1, m_2<M$, what is the probability such that $|S(m_1)-S(m_2)|<\varepsilon$? 

Since $m_1$ and $m_2$ are independent, we can see the two sums as two i.i.d. normal variables, i.e.

$$
\widetilde{S}_{m_1}, \widetilde{S}_{m_2}\overset{i.i.d.}{\sim} N(\mu, \sigma^2).
$$

Then the difference $\widetilde{D}$ is also a normal variable:

\begin{align*}
    \widetilde{D}=\widetilde{S}_{m_1}-\widetilde{S}_{m_2} \sim N(0, 2\sigma^2).
\end{align*}

We have

\begin{align*}
    Pr(|\widetilde{D}|<\varepsilon) =& \Phi(\frac{\varepsilon}{\sqrt{2}\sigma})-\Phi(-\frac{\varepsilon}{\sqrt{2}\sigma}) \\
    =& 2\Phi(\frac{\varepsilon}{\sqrt{2}\sigma}) - 1.
\end{align*}

In reality, $\varepsilon$ is related to the computational precision. Now let us calculate the resolution limit of numerical operations involved.

Suppose the floating data type has $f$ explicit fraction bits --- for a number $x$ in the normal range, this means a resolution of $\varepsilon_0(x) = \Theta(2^{-f}x)$. For BF16, FP16, FP32, FP64, $f=7, 10, 23, 52$ respectively \citep{henry2019leveragingbfloat16artificialintelligence}.

For two dot products to be considered ``different'', they must have different attention scores, i.e. different values after the normalized softmax. Let $s$ be a positive integer. Let $\mathbf{x}\in \mathbb{R}^s$, and $\mathbf{y}=\text{softmax}(\mathbf{x}/\sqrt{d}).$ For some $i\neq j$, we need to know the minimum value of $|x_i-x_j|$ such that $|y_i-y_j|<\varepsilon_0(y)$.

When $y_i=y_j=y$, disturb $x_i-x_j$. Then

$$
\frac{\partial{(y_i-y_j)}}{\partial (x_i-x_j)}=\frac{y}{\sqrt{d}}.
$$

This means that to ensure $y_i\neq y_j$,

$$
|x_i-x_j|\ge \frac{\sqrt{d} \varepsilon_0(y)}{y} =\Theta(2^{-f}\sqrt{d}).
$$

This is the effective resolution for the dot product, i.e. $\varepsilon_s=\Theta(2^{-f}\sqrt{d})$, from the restriction of softmax.

Also, the dot product itself has the resolution limit of $\varepsilon_d=\Theta(2^{-f} \sum_n a_n).$ Therefore the final resolution constraint is

$$
\varepsilon = \Theta\left(2^{-f}\max{\left(\sqrt{d}, \sum_n a_n\right)}\right).
$$

In our hypothesis that no single $a$ dominates, 

$$
 \sqrt {\sum_{n<\lambda(M)}  a_n^2/|\lambda(M)|}\sim \sum_{n<h} {a_n}/h =\Theta(a_{\max}).
$$

Therefore,

$$\varepsilon/\sigma =\Theta(2^{-f}\max(\sqrt{d}, ha_{\max})/\sigma)=\Theta( 2^{-f}\max{(\sqrt{d}/(\sqrt{\lambda(M)}a_{\max}), h/\sqrt{\lambda(M)})}). $$

Consider a simplified case where $a_n=1$. For a RoPE base of 10,000, if we use FP16, the probability of having the same positional value is 5.6\textperthousand ~at 32k context length. If the RoPE base increases to 100,000, then the probability becomes 6.5\textperthousand. For a random \textbf{pair} of distances, this is the probability of the pair having the exact same attention score. This may not seem like a large number, but if we consider that there are $O(M^2)$ pairs of position, this leads to an astonishing 3.5 million pairs of possible position aliasing.

\begin{thm}
    Uniformly and randomly choose distance $m_1\in [0, M)$. Let $p(M)$ denote the probability that there exists a different distance $m_2\in[0, M)/\{m_1\}$ such that $|S(m_1)-S(m_2)|<\varepsilon$, where $\varepsilon=\Theta\left(2^{-f}\max{\left(\sqrt{d}, \sum_n a_n\right)}\right)$ is the absolute resolution limit for the datatype with $f$ explicit fraction bits. We have $p(M)\ge 1-(1-E)^{M-1}$ and $\lim _{M\to+\infty} p(M)=1$.
\end{thm}

\begin{proof}
For a random $m_1$ in a given context length $M$, the target probability is

$$
p(M)=1-(1-Pr(|\widetilde{D}|<\varepsilon))^{M-1}.
$$

Since

$$
Pr(|\widetilde{D}|<\varepsilon)\ge 2\Phi\left( 2^{-f} h/\sqrt{\lambda(M)}\right)-1\ge 2\Phi( 2^{-f} \sqrt{h})-1,
$$

let $E=2\Phi( 2^{-f} \sqrt{h})-1$. Then

$$
p(M)\ge 1-(1-E)^{M-1}.
$$

Since $p(M)\le 1$, we have

$$
\lim_{M\to+\infty}p(M) = 1.
$$
\end{proof}

\begin{table}[ht]
\centering
\caption{ Prob. of positional aliasing for a single pair when $\mathbf{a}=\mathbf{1}$. }

\begin{tabular}{ccccc}
\toprule
\multirow{2}*{M} & \multirow{2}*{B}  & Prob. at bf16 & fp16 & fp32   \\
\cline{3-5}
& & \multicolumn{3}{c}{Expectations} \\
\hline
\multirow{6}*{1024} & \multirow{2}*{10000} & 0.057 & 0.0071 & 8.7e-07 \\
	 & 	 & 3e+04 $\pm$ 2e+02 & 3.7e+03 $\pm$ 6e+01 & 0.46 $\pm$ 0.7 \\
\cline{3-5}
	 & \multirow{2}*{100000} & 0.064 & 0.008 & 9.7e-07 \\
	 & 	 & 3.3e+04 $\pm$ 2e+02 & 4.2e+03 $\pm$ 6e+01 & 0.51 $\pm$ 0.7 \\
\cline{3-5}
	 & \multirow{2}*{1000000} & 0.069 & 0.0087 & 1.1e-06 \\
	 & 	 & 3.6e+04 $\pm$ 2e+02 & 4.5e+03 $\pm$ 7e+01 & 0.56 $\pm$ 0.7 \\
\hline
\multirow{6}*{4096} & \multirow{2}*{10000} & 0.052 & 0.0065 & 8e-07 \\
	 & 	 & 4.4e+05 $\pm$ 6e+02 & 5.5e+04 $\pm$ 2e+02 & 6.7 $\pm$ 3e+00 \\
\cline{3-5}
	 & \multirow{2}*{100000} & 0.058 & 0.0073 & 8.9e-07 \\
	 & 	 & 4.9e+05 $\pm$ 7e+02 & 6.1e+04 $\pm$ 2e+02 & 7.4 $\pm$ 3e+00 \\
\cline{3-5}
	 & \multirow{2}*{1000000} & 0.064 & 0.008 & 9.7e-07 \\
	 & 	 & 5.4e+05 $\pm$ 7e+02 & 6.7e+04 $\pm$ 3e+02 & 8.2 $\pm$ 3e+00 \\
\hline
\multirow{6}*{16384} & \multirow{2}*{10000} & 0.048 & 0.006 & 7.4e-07 \\
	 & 	 & 6.5e+06 $\pm$ 2e+03 & 8.1e+05 $\pm$ 9e+02 & 9.9e+01 $\pm$ 1e+01 \\
\cline{3-5}
	 & \multirow{2}*{100000} & 0.054 & 0.0068 & 8.3e-07 \\
	 & 	 & 7.3e+06 $\pm$ 3e+03 & 9.1e+05 $\pm$ 1e+03 & 1.1e+02 $\pm$ 1e+01 \\
\cline{3-5}
	 & \multirow{2}*{1000000} & 0.059 & 0.0074 & 9.1e-07 \\
	 & 	 & 8e+06 $\pm$ 3e+03 & 1e+06 $\pm$ 1e+03 & 1.2e+02 $\pm$ 1e+01 \\
\hline
\multirow{6}*{32768} & \multirow{2}*{10000} & 0.047 & 0.0058 & 7.1e-07 \\
	 & 	 & 2.5e+07 $\pm$ 5e+03 & 3.1e+06 $\pm$ 2e+03 & 3.8e+02 $\pm$ 2e+01 \\
\cline{3-5}
	 & \multirow{2}*{100000} & 0.052 & 0.0065 & 8e-07 \\
	 & 	 & 2.8e+07 $\pm$ 5e+03 & 3.5e+06 $\pm$ 2e+03 & 4.3e+02 $\pm$ 2e+01 \\
\cline{3-5}
	 & \multirow{2}*{1000000} & 0.057 & 0.0071 & 8.7e-07 \\
	 & 	 & 3.1e+07 $\pm$ 5e+03 & 3.8e+06 $\pm$ 2e+03 & 4.7e+02 $\pm$ 2e+01 \\
\hline
\multirow{6}*{65536} & \multirow{2}*{10000} & 0.045 & 0.0056 & 6.9e-07 \\
	 & 	 & 9.7e+07 $\pm$ 1e+04 & 1.2e+07 $\pm$ 3e+03 & 1.5e+03 $\pm$ 4e+01 \\
\cline{3-5}
	 & \multirow{2}*{100000} & 0.051 & 0.0063 & 7.7e-07 \\
	 & 	 & 1.1e+08 $\pm$ 1e+04 & 1.4e+07 $\pm$ 4e+03 & 1.7e+03 $\pm$ 4e+01 \\
\cline{3-5}
	 & \multirow{2}*{1000000} & 0.055 & 0.0069 & 8.4e-07 \\
	 & 	 & 1.2e+08 $\pm$ 1e+04 & 1.5e+07 $\pm$ 4e+03 & 1.8e+03 $\pm$ 4e+01 \\
\bottomrule
\end{tabular}
\end{table}

\subsection{Token Inversion}
\label{subsec:prob-inv-imp}
For the token identification objective, we want to consider how relevant the key token (corresponding to the given $\mathbf{k}$) is to the query token (corresponding to the given $\mathbf{q}$). 

\begin{definition}
\label{def:token_inversion}
    For a query vector $\mathbf{q}$ and two key vectors $\mathbf{k}_1,\mathbf{k}_2$, if $S_{\mathbf{q},\mathbf{k}_1}(0)<S_{\mathbf{q},\mathbf{k}_2}(0)$, but for some $m>0$, $S_{\mathbf{q},\mathbf{k}_1}(m)>S_{\mathbf{q},\mathbf{k}_2}(m)$, then a \textbf{token inversion} occurs at $m$.
\end{definition}

For the following analysis, we apply an important simplification to the definition. Since for the same $(\mathbf{q},\mathbf{k})$ pair, the theoretical upper bound of the RoPE product, denoted by $S_{\max}$, is

$$\sum_{n=0}^{h-1}a_j\cos(m\theta^n+\phi_n)\leq\sum_{n=0}^{h-1}a_n\equiv S_{\max},$$

for the pair $(\mathbf{q},\mathbf{k})$, we may introduce a hypothetic key token, called the ``prime'' token, whose key vector is denoted by $\mathbf{k}'$, that reaches this RoPE product at $m=0$. That is to say, The pair $(\mathbf{q},\mathbf{k}')$ shares the same set of amplitudes, $a_j$, as $(\mathbf{q},\mathbf{k})$, with all initial phase biases $\phi_j=0$. Now, if we randomly select $0\le m<M$ , we wish to see the probability that this ``prime'' key token has a smaller RoPE product than a key token at distance $m$.

Formally, for a given pair $(\mathbf{q},\mathbf{k})$, let the ``prime'' key vector be $\mathbf{k}'$. Randomly and uniformly select integer $m$ between $[0, M)$. We assume that $\{\phi_n\}$ are not all 0, since otherwise $\mathbf{k}=\mathbf{k}'$. We also assume $a_{\max}>0$, otherwise either $\mathbf{q}$ or $\mathbf{k}$ is zero.

Using the conclusion in \cref{sec:rope-sum-as-normal}, the difference is

    \begin{align*}
D=&\sum_{n=0}^{h-1}a_n\cos(m\theta^n+\phi_n)-\sum_{n=0}^{h-1}a_n\cos(m\theta^n)\\
=-2&\sum_{n=0}^{h-1} a_n\sin(\phi_n/2) \sin(m\theta^n+\phi_n/2)
\end{align*}

Let $A_n=-2a_n\sin(\phi_n/2)$, $\varphi_n=\phi_n/2-\pi/2$. Then
$$
D=\sum_{n=0}^{h-1} A_n\cos(m\theta^n+\varphi_n).
$$

This form is exactly what is studied in \cref{sec:rope-sum-as-normal}. Just like any RoPE product,
$D$ can also be seen as a normal variable $\widetilde{D}\sim N(\mu, \sigma)$. 

\begin{thm}
    The probability $Pr\left(\widetilde{D}>0\right)$ satisfies
\begin{align} 
    Pr\left(\widetilde{D}>0\right)=& \Phi\left(\frac{\mu} \sigma\right)\label{eq:second_prob},
\end{align}
    where

\begin{align*}
    \mu=&\sum_{n\ge \lambda(M)} a_n(\cos \phi_n-1) ,\\
    \sigma=&\sqrt{\sum_{n<\lambda(M)} a_n^2 \sin^2 \phi_n}.
\end{align*}
When $M=\Theta(B)$ and $\lambda(M)=h$, $Pr\left(\widetilde{D}>0\right)=1/2$.
\end{thm}

\begin{proof}
$\mu$ and $\sigma$ can be obtained by applying \cref{eq:normal_mu} and \cref{eq:normal_sigma} on $\{A_n\}$ and $\{\varphi_n\}$. The probability then follows that for normal distribution. 

When $\lambda(M)=h$, $\mu=0$. Since we assume that $\{\phi_n\}$ are not all 0 and $\{a_n\}$ are not all 0, $\sigma\neq 0$, giving $Pr\left(\widetilde{D}>0\right)=1/2$.
\end{proof}

\subsection{Token Aliasing}
\label{subsec:token_aliasing_formal}
\begin{definition}
\label{def:token_aliasing}
    For a query vector $\mathbf{q}$ and two key vectors $\mathbf{k}_1,\mathbf{k}_2$, if the numerical results of RoPE product under a certain datatype $\hat{S}_{\mathbf{q},\mathbf{k}_1 | \texttt{dtype}}(m)=\hat{S}_{\mathbf{q},\mathbf{k}_2 | \texttt{dtype}}(m)$ for some $m$, then a \textbf{token aliasing} occurs at $m$.
\end{definition}

For the purpose of simplicity, we make an assumption that in a fixed transformer head, for all pairs $(\mathbf{q},\mathbf{k})$, for all $0\le n<h$, the magnitude at the given dimension pair $a_n(q,k)={|(q_{2n},q_{2n+1})||(k_{2n}, k_{2n+1})|}$ stays the same. For real attention heads, we may make the assumption that for a fixed frequency component $n$ and randomly chosen $q,k$ pairs, $a_n$ follows some normal distribution, and we may take its mean as $a_n$ discussed here.

If for query vector $\mathbf{q}$ and key vectors $\mathbf{k}_1$ and $\mathbf{k}_2$, 

$$
S_1(m)=\langle \mathbf{q},\mathbf{k}_1\rangle _m= \sum_{n=0}^{h-1}a_n \cos(m\theta^n+\phi_{1,n})
$$

and

$$
S_2(m)=\langle \mathbf{q},\mathbf{k}_2\rangle _m= \sum_{n=0}^{h-1}a_n \cos(m\theta^n+\phi_{2,n})
$$

are independent, then for $\varepsilon > 0,$ the probability of token aliasing can be expressed as

$$
Pr(|S_1(m)-S_2(m)|<\varepsilon).
$$

The difference $D=S_1(m)-S_2(m)$ can be seen as a normal variable

$$\widetilde{D}\sim N(\mu_1-\mu_2, \sigma_1^2+\sigma_2^2).$$

\begin{thm}
If $\sum_n{a_n}/h \approx \sqrt{\sum_n{a_n^2}/h} \approx a_{\max}$, then 

$$Pr(|\widetilde{D}|<\varepsilon)\approx 2\frac{2^{-f}h}{\sqrt{\lambda(m)}}\textrm{pdf}\left(\frac{h}{\sqrt{\lambda(m)}}-\sqrt{\lambda(m)}\right),$$

where pdf is the probability density function of standard normal distribution,

$$
\textrm{pdf}(x)=\frac{1}{\sqrt{2\pi}} e^{-x^2/2}.
$$ 
\end{thm}

\begin{proof}
Without loss of generality, assume $\mu_1>\mu_2$. We have

$$
\mu=\mu_1-\mu_2=\sum_{n\ge \lambda(m)} a_n(\cos \phi_{1,n}-\cos \phi_{2,n})\le 2\sum_{n\ge \lambda(m)} a_n,
 $$

$$
\sigma^2=\sigma_1^2+\sigma_2^2=2\sigma_1^2=\sum_{n< \lambda(m)} a_n^2.
$$

So 

\begin{align*}
    Pr(|\widetilde{D}|<\varepsilon)=\Phi\left(\frac{\varepsilon-\mu}{\sigma}\right)-\Phi\left(\frac{-\varepsilon-\mu}{\sigma}\right).
\end{align*}

Since no $a_n$ dominates, using the same estimations for position aliasing,

\begin{align*}
\mu/\sigma=&\Theta(\frac{h}{\sqrt{\lambda(m)}}-\sqrt{\lambda(m)}),\\
\varepsilon/\sigma\ge&\Theta(\frac{2^{-f}h}{\sqrt{\lambda(m)}}),\\
    \Phi\left(\frac{\varepsilon-\mu}{\sigma}\right)-\Phi\left(\frac{-\varepsilon-\mu}{\sigma}\right)\approx & 2\frac{2^{-f}h}{\sqrt{\lambda(m)}}\text{pdf}(\frac{h}{\sqrt{\lambda(m)}}-\sqrt{\lambda(m)}).
\end{align*}
\end{proof}

When $\lambda(m)=h$, the probability converges to $\Theta(2^{1-f}\sqrt{h}/\sqrt{2\pi})$.

\section{Experiment Details}
\label{sec:exp_det}
\subsection{Case Study}
\label{subsec:case_study_details}

We use Head 0, Layer 0 of Llama3.1-8B \citep{grattafiori2024llama3herdmodels} as the case study sample. However, our method is applicable to any head in any layer of any RoPE-based decoder transformer model.

Each failure mode features a query token and one to two key tokens. For each key token, we calculate its RoPE product with the query, $S(m)$, for every $m$ in the context range $(0, M]$. 
 We calculate $S(m)$ for a certain head using the following steps:
\begin{itemize}
    \item we construct an input following the format \texttt{<bos>}\footnote{The \texttt{<bos>} token is added for the input sanity and is crucial if the target head is not from the first layer.}\texttt{ [key] ... [key] [query]}, where the key token is repeated $M$ times. 
    \item We calculate the hidden states for the target layer. This involves a forward pass through the embedding layer and every transformer layer before the one which contains our selected head. This standard process can be accelerated by FlashAttention \citep{dao2022flashattention}.
    \item For the selected head, we obtain the \textit{un-normalized} attention score, i.e. with neither the $\sqrt{d}$ normalization nor softmax. We only calculate the final row of the attention score matrix, since we are interested in the RoPE products involving the query token. 
    \item The result, $\bm{o}$, is an $(M+2)$-d array starting at index 0. For $0<m\le M$, we have $S(m)=o_{M+1-m}$. 
\end{itemize}

The occurrences of the four failure modes are then identified using the definitions \ref{def:position_inversion}, \ref{def:position_aliasing}, \ref{def:token_inversion}\footnote{For simplicity of implementation, we use $S(1)$ instead of $S(0)$ when identifying token inversion.}, \ref{def:token_aliasing} in \cref{sec:failure_modes}.  

For larger models or heads not located in the first layer, GPUs with memory corresponding to the model size are recommended. However, for our case study, we only load the first layer and the whole case study is conducted using a lap-top grade CPU with 64GB memory.

\subsubsection{Aliasing Probs for FP16 in Case Study}
FP16 has 10 explicit fraction bits, and the aliasing probabilities are different from using BF16. For position aliasing, see \cref{fig:pos_aliasing_example_fp16}.

\begin{figure*}[tbhp]
    
    \centering
    \begin{subfigure}[t]{0.46\textwidth}
        \centering
        \includegraphics[width=\textwidth]{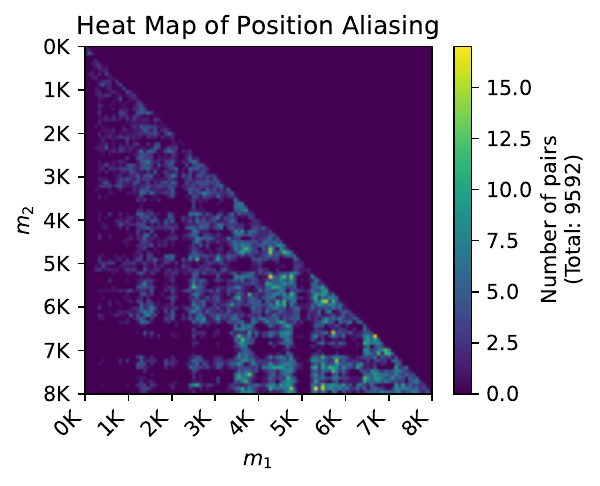}
        \caption{Distribution of position aliasing pairs for key ``cat'' and query ``pet''. 
        }
        \label{fig:position_aliasing_example_cat}
    \end{subfigure}
    ~ 
    \begin{subfigure}[t]{0.46\textwidth}
        \centering
        \includegraphics[width=0.98\textwidth]{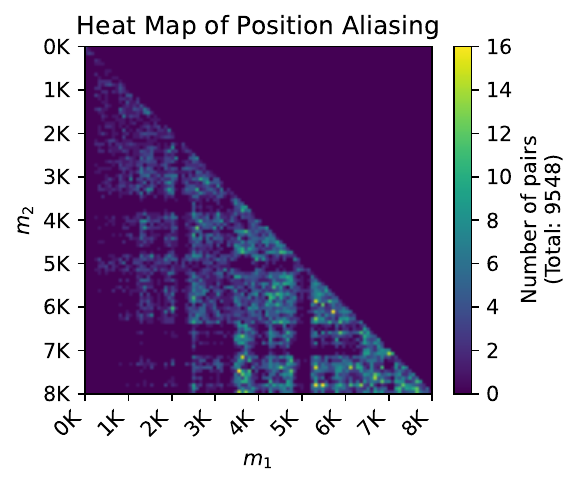}
        \caption{Position aliasing pairs for key ``dog'' and query ``pet''.}
        \label{fig:position_aliasing_example_dog}
    \end{subfigure} 
    \centering
    \begin{subfigure}[t]{0.5\textwidth}
        \centering
        \includegraphics[width=0.99\textwidth]{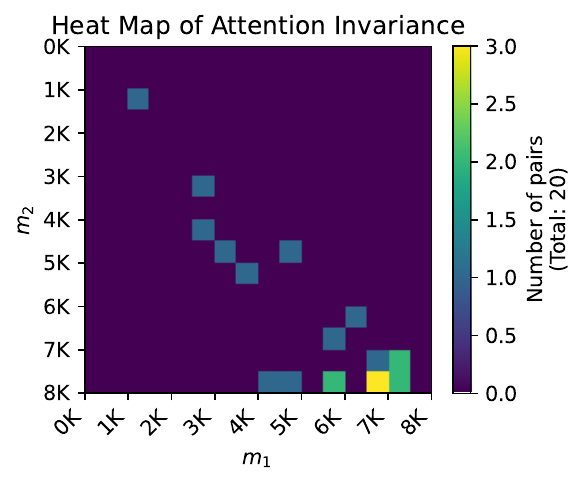}
        \caption{Attention invariance pairs for ``cat'', ``dog'' and ``pet''. 
        }
        \label{fig:attention_invariance_example_fp16}
    \end{subfigure}
    \caption{Heat maps of position aliasing and attention invariance pairs. FP16, Llama3.1-8B, Layer 0 Head 0. Pairs are grouped into a total of $200\times 200$ bins for position aliasing, and $16\times 16$ bins for attention invariance. 1K = 1024.}
\label{fig:pos_aliasing_example_fp16}
\end{figure*}

For token aliasing, see \cref{fig:token_aliasing_fp16}.

\begin{figure*}[htbp]
    
\centering
        \includegraphics[width=0.6\textwidth]{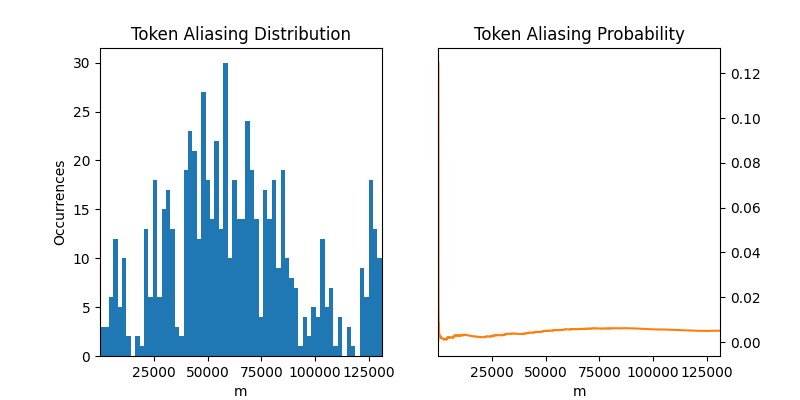}
        \caption{Distribution and probability of token aliasing for keys ``cat'' and ``dog'' and query ``pet''. Llama3.1-8B, Layer 0 Head 0. The probability converges to what is close to the estimated probability of 0.006.
        }
        \label{fig:token_aliasing_fp16}

\end{figure*}

\subsection{The Indexing Task}

We randomly generate a list \texttt{arr}. The list only contains integers 0, 1, 2 and 3. We control the length of the list to be powers of 2, ranging from 4 to 4096, and use \texttt{f"\{arr\}"} to convert it into a string. The model is required to answer the value of the given index \texttt{arr[i]}. Different models may not tokenize the input or apply the chat template the same way, but the correspondence is roughly 3 tokens per element, and we only evaluate model performance based on token count. For each length, we test the models on 10 randomly generated lists. For each list, we generate 10 independent query sessions with random indices. For each list, we report the average input tokens and mean accuracy across the 10 query sessions. We aggregate the results for different lists by also reporting the mean standard deviation of accuracy. 

We use the following prompt: \texttt{arr = \{arr\}\textbackslash nGiven the above array, don't think and directly answer the corresponding value concisely: arr[\{key\}] = }

We use the following models: smaller models (less than 10 B), including Llama-3.1-8B-Instruct \citep{grattafiori2024llama3herdmodels}, Mistral-7B-Instruct-v0.3 \citep{mistralai_mistral_7b_instruct_v03_2024}, Qwen3-8B \citep{yang2025qwen3technicalreport}; larger models (more than 100B), including DeepSeek-V3.1 \citep{deepseekai2025deepseekv3technicalreport}, Kimi-K2.5 \citep{kimiteam2026kimik25visualagentic}, gpt-oss-120b \citep{openai2025gptoss120bgptoss20bmodel}.

For models less than 10B, we use the HuggingFace transformers implementation \citep{wolf-etal-2020-transformers}. These inferences can be done on a GPU of more than 40GB. For models larger than 100B, we use the TogetherAI API \citep{togetherai_docs}. We disable reasoning mode where possible, and prompt the model to answer directly, but do not limit the generation length. We retrieve the last number from the model generation.

\section{RoPE in Real Models}
\label{sec:real_model_discussions}
As suggested in the experiment section \S\ref{sec:exp}, the  multiple heads and layers in a real model may only offer redundancies that provide limited protection. We provide the following discussions that are open-ended and intuitive rather than evidence-based, which may serve as preliminary insights for future analyses using more realistic models:

First, redundancy across attention heads is limited, since the heads are highly specialized and sparsely activated. Prior work \citep{kahardipraja2025atlasincontextlearningattention, wu2024retrievalheadmechanisticallyexplains, lin2026retrievalheadsdynamic} has shown that different heads are often responsible for text dependencies of different ranges, and that only a small number of heads are active at the same time. As an example, consider a long input with a 50\% probability of position inversion. Even if 8 out of 32 heads are actively retrieving, a proportion larger than what is commonly observed \citep{wu2024retrievalheadmechanisticallyexplains}, once every 250 tokens, \textit{all} retrieval heads will simultaneously exhibit position inversion. 

Second, redundancy across layers is largely limited due to the residue connection. Errors introduced by earlier layers carry on to later ones connected in series. Improved architectures like \cite{chen2026attnres} apply residue connections in parallel. This aggregation strategy potentially reduces the accumulative error. Even in such cases, errors in the final layer never get the chance to be calibrated. An attention invariance failure in the final layer can directly lead to a wrong output token. 

Finally, the underlying mechanism applies beyond the number of heads and layers. As context increases, an increasing oscillation leads to less positional uniqueness, and an increasing decay reduces the value of RoPE product and compresses token-wise difference. This means that the position and token identity of faraway text are increasingly likely to be poorly distinguished, leading to a weak or erroneous contribution to the context-aware representation.

%% file: content/101_related_works.tex
\section{Related Works}

\paragraph{Long Context Models} 

State-of-the-art language models are delivered with increasingly larger context window limits, some of them well beyond 1 million tokens (citations like \cite{comanici2025gemini,llama4,gpt5,magic2024ltm100m}). To deliver and utilize long-context models, apart from improvements in data curation (\cite{fu2024data}), training optimization (\cite{dao2022flashattention,li2023sequence,liu2023ring}), and efficient deployment (\cite{xiao2023streamingllm}), one of the main strategies is to adjust the value of RoPE base, which in the original paper is 10,000, ``the worst base value'' (\cite{liu2023scaling}). 

\paragraph{A War of Increasing RoPE Base}

Rotary Positional Embedding (RoPE, \cite{su2021roformer}) exhibits a decaying effect on attention scores of distant tokens. It is purposefully designed this way to model natural language with decreased dependency over longer distance (\cite{su2021roformer}). However, it is widely believed that this decay at least partially limits long-context performance (\cite{tworkowski2023focused, NEURIPS2024_9f12dd32, zhong2024understandingropeextensionslongcontext, miranda2024round}). Therefore, attempts to extend context length usually feature increased values of RoPE base, which leads to slower decay (\cite{peng2024yarn, gao2025train}). As a variant, some works remove RoPE at least partially from attention calculation \citep{wang2024lengthgeneralizationcausaltransformers,yang2025ropenopeagainnew}, or only apply RoPE to certain dimensions \citep{javaheripi2023phi,deepseek2026v4}, which, according to our analysis, is equivalent to applying an infinitely large RoPE base. However, although the resulting long-context models excel in retrieval-based tasks such as Needle-in-a-Haystack (\cite{kamradt2023needleinahaystack}), they still suffer a substantial performance drop on tasks that involve reasoning, variable tracking or other long-term dependency, well within their context limit (\cite{du2025contextlengthhurtsllm, hsieh2024ruler, kuratov2024babilong, liu-etal-2024-lost}).

\paragraph{RoPE as a Function of Distance}
Multiple works have analyses on the dot product after applying RoPE (the RoPE product) as a function w.r.t. distance between two tokens. Works like \cite{jonasson2025rotaryoffsetfeatureslarge,liu2023scaling, miranda2024round} identify high and low frequencies based on whether a rotary component completes a full circle. Works that  mainly focus on the decaying nature of RoPE, and its effect on the decreased attention score assigned to distant tokens (\cite{NEURIPS2024_9f12dd32,xiong-etal-2024-effective,chen2024fortify}), lead to suggested lower bounds of the RoPE base for certain context lengths (\cite{NEURIPS2024_9f12dd32,peng2024yarn}). However, these only focus on one side of the full picture: there are few systematic discussions about the oscillation in the waveform affected by the high-frequency terms, and how this can lead to upper bounds for the RoPE base selection. As a function of distance, the mathematical properties of the RoPE product are still little studied in depth, let alone their practical implications.

\paragraph{The Long-Context Dilemma} \cite{liu2026rotarypositionalembeddingsphase} is among the first to study the oscillation using signal processing techniques, introducing a theoretical upper-bound where the numeric precision and Nyquist Limit are involved. 
\cite{liu2026rotarypositionalembeddingsphase} shows that as long as we must use transformer with RoPE to process long data, we must face some sort of uncertainty and are forced to choose between positional and token-wise accuracy. It is natural to think that transformers with RoPE have limited capacity to deal with long-context data. One must doubt whether the current ways of utilizing long-context models are the most adequate: are the reasoning models that rely on long chain of thoughts really capable of being responsible to their own thought process? Can chat models really recall distant history? Is it really worth it to train models with context lengths of tens of millions of tokens just to pass the Needle in a Haystack test? It is possible that to make language models really process large volumes of text, we need some other sorts of model structures, or strategies to keep the current models within their effective context range.